\documentclass[10pt]{article} 
\usepackage[preprint]{tmlr}

\usepackage{amsmath,amsfonts,bm}

\def\eqref#1{equation~\ref{#1}}

\def\1{\bm{1}}

\DeclareMathAlphabet{\mathsfit}{\encodingdefault}{\sfdefault}{m}{sl}
\SetMathAlphabet{\mathsfit}{bold}{\encodingdefault}{\sfdefault}{bx}{n}

\DeclareMathOperator*{\argmax}{arg\,max}
\DeclareMathOperator*{\argmin}{arg\,min}

\usepackage{xcolor}
\usepackage{float}
\usepackage{hyperref}
\usepackage{url}
\usepackage{amssymb}
\usepackage{amsthm}
\usepackage{graphicx}
\usepackage{subcaption}

\usepackage{tcolorbox}
\usepackage{enumitem}
\usepackage{xcolor}

\usepackage{makecell}

\usepackage{placeins}
\usepackage{booktabs}

\newtheorem{definition}{Definition}
\newtheorem{assumption}{Assumption}
\newtheorem{theorem}{Theorem}
\newtheorem{lemma}{Lemma}
\newtheorem{remark}{Remark}

\newtheorem*{proofsketch}{Proof sketch}

\usepackage{algorithm}
\usepackage{algpseudocode}

\title{Point-Identification  of a Robust 
 Predictor Under  Latent \\ Shift with Imperfect Proxies}

\author{\name Zahra Rahiminasab \email zahra.rahiminasab@aalto.fi \\
      \addr Department of Computer Science\\
      Aalto University 
      \AND
      \name Reza Soumi \email reza.soumi@aalto.fi\\
      \addr Department of Computer Science\\
      Aalto University
       \AND
      \name Arto Klami \email  arto.klami@helsinki.fi  \\
       \addr Department of Computer Science\\
       University of Helsinki
               \AND
      \name Samuel Kaski \email samuel.kaski@aalto.fi\\
 \addr      ELLIS Institute Finland
      \\
      \addr Department of Computer Science\\
      Aalto University
      \\
      \addr Department of Computer Science\\
      Manchester University
      }

\renewcommand{\argmin}{\operatorname*{arg\,min}}

\def\month{05}  
\def\year{2026} 
\def\openreview{\url{https://openreview.net/forum?id=QFJuVreJDC}} 

\begin{document}

\maketitle

\begin{abstract}

Addressing the domain adaptation problem becomes more challenging when distribution shifts across  domains stem  from latent confounders that affect both covariates and outcomes. Existing proxy-based approaches that address latent shift rely on a strong completeness assumption to uniquely determine (point-identify) a robust predictor. Completeness  requires that  proxies have  sufficient information about variations in latent confounders. For imperfect proxies the mapping from confounders to the space of  proxy distributions is non-injective, and multiple latent confounder values can generate the same proxy distribution. This breaks the completeness assumption and observed data are consistent with multiple potential predictors (set-identified).  To address this, we  introduce latent equivalent classes (LECs). LECs are defined as groups of latent confounders that induce the same conditional proxy distribution. We show that point-identification for the robust predictor remains achievable as long as multiple domains differ sufficiently in how they mix  proxy-induced LECs to form the robust predictor. 
This domain diversity condition 
is formalized as a cross-domain rank condition on the mixture weights, which is substantially weaker assumption than completeness. 
 We introduce the Proximal Quasi-Bayesian  Active learning (PQAL) framework, which actively queries 
 a small, targeted set of diverse domains that satisfy this rank condition. PQAL can  recover the point-identified predictor,  demonstrates robustness  to varying degrees of shift and  outperforms previous methods on synthetic data and semi-synthetic dSprites, IHDP, ACS Folktables datasets.
\end{abstract}

\section{Introduction}

Domain adaptation transfers models trained on labeled source domains to a target domain, where 
labeled data is either unavailable entirely  \citep{tsai2024proxy} or restricted to a few samples (few-shot domain adaptation) \citep{motiian2017few}.  To  successfully adapt a trained predictor from the source to the target domain, some assumptions about the type of  data distribution shift must be made \citep{david2010impossibility}. Traditionally, covariate \citep{shimodaira2000improving} or label \citep{lipton2018detecting} shift assumptions are used. However, these assumptions do not always hold in real-world settings. 
In such settings, let $S$  source domains and a target domain are indexed by $Z \in \mathcal{Z} =\{1,\dots,S,S+1\}$.  Then, the shift in distribution between source and target domains is due to a latent confounder $U \in \mathcal{U}$, such that $P(U|Z=i) \neq P(U|Z=S+1)$  for all $i\in \{1,\dots,S\}$, which affects  the covariates $X \in \mathcal{X} \subseteq \mathbb{R}^{d_X}$ and the outcomes $Y\in \mathcal{Y} \subseteq \mathbb{R}^{d_Y}$ \citep{,alabdulmohsin2023adapting,tsai2024proxy}.

Due to such unobserved confounders $U$, the predictors may learn spurious correlations between $X$ and $Y$. Previous studies, such as \citet{tsai2024proxy,alabdulmohsin2023adapting,prashant2025scalable}, use  proxies $W \in \mathcal{W}$ to address the  latent shift. They assume that the proxy is informative enough to  uniquely identify (point-identify) the parameter or function of interest, based on the probability of the observed data \citep{manski2005partial}. 
For discrete unobserved confounders and  proxies,  being  informative  translates to the  invertibility of the conditional distribution of the proxy  given the unobserved confounder $P(W|U)$, which is often formalized as  a rank condition \citep{alabdulmohsin2023adapting,prashant2025scalable,tsai2024proxy}. In contrast, when  unobserved confounders and proxies are continuous  \citep{tsai2024proxy}, the completeness assumption must hold, which ensures  injectivity of the conditional operator induced by $P(W|U)$\citep{miao2018identifying}.
These assumptions break in practice when proxies are imperfect or insufficiently informative. 
More precisely,  when the $P(W|U)$ is not invertible, or the conditional expectation operator is not injective. In such cases,  parameters or functions of interest can only be partially identified (set-identified).

In human-in-the-loop learning, proxies are not passively observed; instead, they are actively collected through budget-limited interactions with a human expert. As a result, these proxies are often noisy or weak  \citep{slomanproxy}. Examples include  expert annotations \citep{srivastava2020robustness} or feedback \citep{nikitin2022human}.
Since acquiring proxies requires costly interactions, only a limited number of proxies can be obtained in practice.
Consequently, it is important to identify robust predictors while strategically selecting these imperfect  proxy queries under strict budget constraints. We address the aforementioned problems by making the following contributions.

\begin{itemize}
    \item \textbf{Relaxing completeness via latent equivalent classes (LECs):}
    This paper addresses domain adaptation under latent shift, when proxies are imperfect and do not  satisfy the standard completeness assumption. When the completeness assumption does not hold, distinct values of the latent confounder may generate the same conditional proxy distribution $P(W|U)$. We introduce latent equivalent classes (LECs) as  sets of latent confounder values that are observationally indistinguishable  given the proxy.  Although individual latent confounder values cannot be recovered, the conditional proxy distribution in each domain can be represented as  a mixture over these LECs. This decomposition enables  tracking  latent shifts by leveraging  variation in mixture weights across domains.
    \item \textbf{Point-identification via distinguishing environment set:}
    Based on the obtained mixture formulation, we prove that the robust predictor ($\mathbb{E}[Y|X=x,Z=S+1]$, where  $X$ is a covariate and  $Y$ is an outcome. In addition, environments are indexed by discrete variable $Z$.)
    is point-identified as long as the observed domains exhibit sufficient diversity in their mixture weights over LECs. We call such environments distinguishing environments $Z^{\star} \subseteq \mathcal{Z}$. This condition is formalized as a simple cross-domain rank condition. We formally prove that this rank condition is strictly weaker than the completeness assumption, as it depends on variation across environments rather than on the injectivity of the conditional operator in each environment.
    \item \textbf{Proximal Quasi-Bayesian Active Learning (PQAL):}
      We propose the Proximal Quasi-Bayesian Active Learning (PQAL) framework to train a point-identified robust predictor.
    PQAL actively queries proxies to satisfy the required rank condition.  We show  that PQAL  recovers a point-identified predictor and achieves lower mean squared error than  state-of-the-art methods across varying degrees of latent distribution shift on synthetic and semi-synthetic dSprites, IHDP, ACS Folktables datasets.  
\end{itemize}

\section{Problem formulation}\label{sc:Prob}

Let  $ \mathcal{X}$ denotes  the space of observed covariates $X$, and $ \mathcal{Y}$
the space of outcomes $Y$.
We define a set of $S$ sources and a target environment with  domain  IDs $Z \in \{1,\dots, S, S+1\}$. 
In every environment, both the covariates and the outcomes are influenced by an unobserved confounder $U$, which varies across the environments. That is, $p(U|Z=z)$. This distribution is also unknown, but we assume that partial knowledge about it can be obtained via a proxy $W$ that  contains indirect and imperfect information regarding changes in the distribution of unobserved confounders $U$. 
When proxies are imperfect, multiple latent values from $U$ can induce the same conditional proxy distribution $P(W\mid U)$. 
For instance, periodically repeating confounders, like orientation where $2\pi + \theta$  and $\theta$ induce the same proxy distribution (see dSprites dataset in the Section \ref{sec:data} for an example, where we consider domains that differ by such orientation).

\begin{assumption}[Latent shift (from \citet{tsai2024proxy})]
Latent shift is translated to the following  conditions:
\begin{enumerate}
     \item The shift between source distribution $P$ and a target distribution $Q$ stems from unobserved confounder $U$. Therefore $P(U) \neq Q(U)$. Also, $\forall z_i,z_j \;\text{when}  \; i \neq j$, $P(U|z_i) \neq P(U|z_j) \neq Q(U)$.

    \item Invariant distributions: 
    a conditional distribution of observed covariates, proxy, and an outcome for an unobserved confounder remain invariant between source and target domains $P(R|U)=Q(R|U)$ for $R \subseteq \{X,Y,W\}$. 
\end{enumerate}
\label{ass: shift}
\end{assumption}
Figure \ref{fig:CAD} illustrates the causal DAG that complies with the latent shift (Assumptions \ref{ass: shift}).

\begin{figure}
    \centering
    \includegraphics[width=0.3\linewidth]{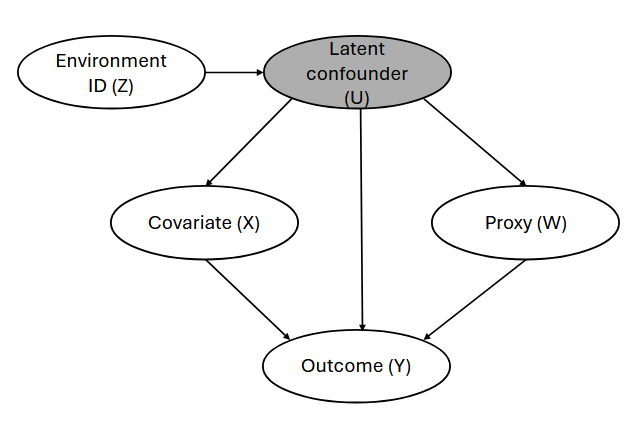}
    \caption{Causal diagram that complies with the latent  shift assumption}
    \label{fig:CAD}
\end{figure}

\paragraph{Problem:} Given $m$ labeled data   from $S$ source domains with observed variables $\{(x_i,w_i,z_i,y_i)\}_{i \in \{1,\dots,m\}}$,  where $W$ is imperfect proxy for the latent confounder $U$. 
We  observe variables $\{(x_j,z_j)\}_{j \in \{1,\dots,n\}}$ and we may also observe  $w_j$  from target domain ($S+1$).  The goal is to uniquely determine (Point-identify) the optimal predictor $\mathbb{E}[Y|X=x,Z=S+1]$ on  a target domain based on the joint distribution of the observed data.
\paragraph{Learning setup:}
Proxies are usually obtained through costly interactions, such as expert feedback.  Therefore, only a limited number of proxies  can be queried from a pool of candidate samples. This practical constraint motivates the use of an active selection strategy  to query the most informative proxies for  learning a point-identified robust predictor.

\begin{tcolorbox}
    [colback=blue!5!white,
    colframe=blue!75!black,
    title = \textbf{Motivating Example: Robust Drone Design Under Manufacturing shift},
    boxrule=0.6pt,
    arc=2mm,
    left=2mm, right=2mm, top=1mm,bottom=1mm,
    before skip=2mm, after skip=2mm,
    title filled=false,
    fonttitle= \bfseries,
    ]
    \setlength{\parskip}{0pt}
    \setlength{\itemsep}{1pt}
    \setlength{\topsep}{1pt}
    \setlength{\parsep}{0pt}
    \setlength{\partopsep}{0pt}

    \textbf{Setting:}
    Consider $S$ different manufacturers (indexed by $Z$) constructing drones characterized by covariates $X$ (e.g., battery mass, wing chord). Both the covariates $X$ and the function that maps them to observable properties $Y$ (e.g. flight speed and hovering time) depend on unobservable confounders $U$, such as manufacturing quality.
    While $U$ cannot be observed directly, we assume a proxy variable $W$ to provide partial and indirect information about it. For example, a human expert inspecting the drones can estimate the overall quality of all manufacturers, but does this in a crude manner: If they cannot differentiate between the quality of some manufacturers they provide the same assessment for them.

    \textbf{Goal:} Train a point-identifiable predictor estimating the properties $Y$ for a new manufacturer, based on the $\{X, Y, W\}$ triplets for the $S$ manufacturers and the proxy variable $W$ for the new manufacturer.
    
\end{tcolorbox}

\section{Notation and Background}

\textbf{Notation:} For any variable $\mathcal{V} \in \{\mathcal{X},\mathcal{W},\mathcal{Z},\mathcal{Y}\}$, let $k: \mathcal{V} \times \mathcal{V} \rightarrow \mathbb{R}$ be a positive semi-definite kernel function, and for every $v \in \mathcal{V}$, feature maps are defined as $\phi(v)= k(v,.)$. The Reproducing Kernel Hilbert Space (RKHS)  \citep{murphy2012machine} is defined on $\mathcal{V}$ corresponding to defined kernel functions. $\mathcal{K}_{V} =[k(v_i,v_j)]$ are gram matrices and $\mathcal{K}_{VV'} =[k(v_i,v'_j)]$ are cross-gram matrices. Also, $\Phi_V(v')=[k(v_1, v')\dots k(v_{|V|},v')]$. In addition, $\otimes$, $\bar{\otimes}$, and $\odot$ are the tensor product, column-wise Khatri-Rao product, and Hadamard products, respectively. 
$\mathcal{D}_{\mathcal{LB}}=\{(x_i,w_i,z_i,y_i)\}^{m_1}_{i=1}$ is a set of labeled samples from the source domains. Also, $\mathcal{D}_{\mathcal{PL}}=\{(x_j,z_j)\}_{j=1}^n$ is a pool of candidate samples without proxies. In addition,  $\mathcal{D}_{\mathcal{PR}}^{r}=\{(x^{(r)}_j,w^{(r)}_j,z^{(r)}_j)\}^{n_{\mathcal{PR}}}_{j=1}$ is a set of queried samples with proxy  at  round $r$ from  source and  target environments. Finally, $\mathcal{D}_{\mathcal{TLB}}^{r}=\{(x^{(r)}_k,w^{(r)}_k,z^{(r)}_k,y^{(r)}_k)\}^{n_{\mathcal{TLB}}}_{k=1}$  is a set of labeled samples from a target environment where $n_{\mathcal{TLB}} \ll n_{\mathcal{PR}}$.
 Throughout   the paper, $\lambda$ with different subscripts denotes regularization hyperparameters.

Following \citet{tsai2024proxy}, we   adapt  identification techniques  from  proximal causal inference (PCI) to provide point-identification guarantees for predictors.
In addition, we use the expected predictive information gain (EPIG)  to design an acquisition function that selects the most informative  queries for distinguishing LECs.  Therefore, we review PCI and EPIG in this section.

\subsection{Proximal causal inference} \label{scc:PCI}
Proximal causal inference (PCI) \citep{tsai2024proxy} uses proxies to adjust for the effect of unobserved confounders.
  PCI solves linear inverse problem $\mathbb{E}[Y-h(X,W)|X=x,Z=z]=0$, utilizing a conditional expectation operator ($\mathcal{T}$) to map a bridge function $h$ to the domain specific predictor ($\mathbb{E}[Y|X=x,Z=z]$) \cite{wang2021scalable,bennett2023inference}. The bridge function $h$  connects covariates $X$ and proxies $W$ to outcome $Y$ by integrating over the unobserved confounder $U$. Formally, PCI estimates the optimal predictor $\mathbb{E}[Y|X=x,Z=z]$ based on a bridge function $h$ and conditional distribution $p(w|x,z)$ by integrating over unobserved confounders $u$ as 
\begin{equation}
\begin{aligned}
&\mathbb{E}[Y|X=x,Z=z]=\int_{\mathcal{U}}\int_{\mathcal{W}}  h(x,w) p(w|u)p(u|x,z) dwdu=\int_{\mathcal{W}} h(x,w) p(w|x,z) dw = \mathcal{T} h.
\end{aligned} 
\label{eq:PCI}
\end{equation}

Establishing the point-identification of a predictor $\mathbb{E}[Y|X=x,Z=S
+1]$ in target environment based on Equation \ref{eq:PCI}  generally involves two steps \citep{miao2018identifying,tsai2024proxy}:
\begin{enumerate}
\item Establishing the existence of a bridge function based on solvability and regularity assumptions (e.g., Picard/range conditions  and square integrability)
\item Point-identification of  optimal predictor  based on the existence of a bridge and by assuming completeness.
\end{enumerate}

 The completeness assumption ensures  the injectivity of the conditional expectation operator $\mathcal{T}$ by  enforcing the following conditions \citep{melnychuk2024quantifying}:
 \begin{enumerate}
     \item Given square integrable function $l$, $\mathbb{E}[l(U)|x,z]=0$ for all $(x,z)\in \mathcal{X} \times \mathcal{Z}$ if and only if $l(U)=0$ almost surely.
     \item Given square integrable function $g$, $\mathbb{E}[g(Z)|x,w]=0$ for all $(x,w)\in \mathcal{X} \times \mathcal{W}$ if and only if $g(Z)=0$ almost surely.
 \end{enumerate}

 In a discrete setting, the first completeness condition corresponds to  invertibility of the matrix representation of $P(U\mid X,Z)$, which captures how latent variables vary across covariates and environments. The second completeness condition translates to invertibility of the matrix representation of  $P(W|U)$, which determines the amount of information proxy carries about the latent confounder. Our framework addresses the failure of the proxy mechanism when  the second completeness condition is violated by introducing a cross-domain rank condition based on variation in latent distributions across environments.
As our theoretical framework is formulated on general measurable (Borel) spaces, 
  the same arguments apply when the variables are discrete, with integrals replaced by finite sums and operators  represented as matrices.

To solve this inverse problem,   the bridge function $h$ and the conditional distribution $p(w|x,z)$ are estimated by  Kernel Proxy Variable (KPV) \citep{mastouri2021proximal}.
KPV  learns the conditional mean operator  $\mathcal{T}$ in the reproducing kernel Hilbert Space (RKHS) via 
$\mathcal{T}= \phi(X) \otimes\mu_{W|X=x,Z=z}$, where $\mu_{W|X=x,Z=z}=\mathbb{E}[\phi(W)|x,z]$ is conditional mean embedding (CME). Then,  by combining learned bridge function and CMEs, the optimal predictor $\mathbb{E}[Y|X=x, Z=z]$ is trained.

\subsection{Expected predictive information gain (EPIG)}
Expected predictive information gain (EPIG) selects candidate samples $x$ that maximize expected reduction in predictive uncertainty at a target input $x_{\star} \sim p_{\star}(x_{\star})$ \citep{smith2023prediction}.  EPIG is defined  based on the conditional  mutual information ($I$) function  between candidate label $y$ and target label $y_*$ as follows:
\begin{equation}
    EPIG(x) =  \mathbb{E}_{p_*(x_{*})}[I(y;y_*|x,x_*)]
\end{equation}

\section{Point-identification   of  predictor}

In this section,  we  formalize the relationship between an unobserved confounder $U$ and a proxy $W$ by introducing the concept of imperfect proxies. Then, we establish  theoretical results for the point-identifiability of  $\mathbb{E}[Y|X=x,Z=S+1]$ despite imperfect proxies.

\subsection{Imperfect proxies and LECs}

First, we  formalize the relationship between an unobserved confounder $U$ and a proxy $W$ by introducing the concept of imperfect proxies.

\begin{definition} [Imperfect proxies]
    The proxy $W$ is imperfect in capturing the latent confounder $U$ if there is more than one value of a latent variable $U$ that can induce its conditional distribution $P(W|U)$.
    Formally, there exists $u \neq u'$ such that:
    \begin{equation}
        P(W\in A|U=u) =  P(W\in A|U=u'), \quad \forall A \in \mathcal{A} \quad \text{where\;} (\mathcal{W},\mathcal{A}) \text{\; is the measurable space of \;} W
    \end{equation} 
     \label{df:Imp}
\end{definition}

Imperfect proxies form groups of latent values that are  indistinguishable by proxy $W$. We formalize these groups as latent equivalent classes (LECs).

\begin{definition}[Latent equivalent classes (LECs)]
     Two values $u$ and $u'$  of latent variable $U$ belong to the same  latent equivalent  class $O$ with respect to $W$ ($u \sim_{O} u'$ ) if and only if  $P(W \in A| U=u) = P(W \in A| U=u') (\forall A \in \mathcal{A}) $.  Each latent equivalence class $
     O_k= \{u:u\sim_{O_k} u'\}$ contains all latent confounder values that induce the same conditional proxy distribution. As a result, proxy $W$ cannot distinguish values within the same LEC; however, it can distinguish across different LECs. Therefore, LECs form a partition of the  entire unobserved confounder space $\mathcal{U}$.
 Formally, we have 
 \begin{equation}
        \mathcal{U}= \bigcup_{j=1}^{|O|} O_j \text{and} \; O_i \cap O_j =\varnothing    (\forall i\neq j)
    \end{equation}
     \label{df:LECS}
\end{definition}

The following remark ensures the proxy does not collapse all latent confounder values into a single LEC. Otherwise proxy has no information about domain-induced changes in $U$, and the domain adaptation problem is trivial and unsolvable.

\begin{remark}[Non-degenerate  LECs]\label{RM:NoNLEC}
     To avoid trivial domain adaptation problem, we assume proxy $W$ is non-degenerate, meaning it 
     induces at least two distinct latent classes ($|O| \geq 2$). 
\end{remark}

\subsection{Theoretical results}\label{sc:Theory}

 Lemma \ref{lm:CMEDEC}    shows that  the conditional  distribution $P(W \mid X=x,Z=z)$ 
can be decomposed as a mixture over LECs. Establishing this decomposition is a key step for the point-identification of the  predictor $\mathbb{E}[Y|X=x, Z=S+1]$.
 To ensure the validity of our derivations, the  following regularity conditions must hold.

\begin{assumption} [Regularity condition]
We assume the domains $\mathcal{X}$, $\mathcal{Y}$, $\mathcal{W}$,  $\mathcal{Z}$, and $\mathcal{U}$ are  Borel spaces with $\sigma$-algebra. Also, the  conditional distribution $P(W|U)$  exists. In addition,  LECs $\{O_j\}_{j=1}^{\mid O \mid }$ are measurable sets with respect to the $\sigma$-algebra of $U$. 
  \label{ass:Reg} 
\end{assumption}
\begin{lemma}[Decomposition of Conditional Distribution]
Let $\{O_j\}_{j=1}^{|O|}$ be the latent equivalent classes (LECs) defined in Definition \ref{df:LECS}. Assume the regularity condition of Assumption \ref{ass:Reg} holds. For each LEC $O_j$, define  its conditional distribution   $P_j(A)=P(W\in A \mid U\in O_j)$, the corresponding mixing weight $\pi_j(x,z):=P(U \in O_j\mid X=x,Z=z)$.
Under latent shift assumption (Assumption \ref{ass: shift}), and its corresponding causal DAG (Fig \ref{fig:CAD}), $W \perp (X,Z) \mid U$ holds. This implies that  the conditional distribution of $W$, has the mixture representation
 \begin{equation}
     P(W \in A|X=x,Z=z)= \sum_{j=1}^{|O|} \pi_j(x,z) P_j(
     A
     ) \quad \text{for any measurable set\;} A \subseteq \mathcal{W}.
 \end{equation}
 
\begin{proof}
    Refer to Appendix \ref{sc:lmCMEDEC} for the proof.
\end{proof}
\label{lm:CMEDEC}
\end{lemma}

If the mixture weights in Lemma \ref{lm:CMEDEC} do not vary across environments, then the  proxy distribution $P(W|X,Z)$ provides no information to distinguish among different LECs. Therefore, the mixture weights must be linearly independent across environments. We formally define such environments as follows:
\begin{definition} [Distinguishing environment set]
Let $Z^{\star}=\{z_1,\dots,z_e\}$ be a set of $e$ source   environments. Let  $\pi_j(x,z)$ be  mixture weights   defined in Lemma \ref{lm:CMEDEC}. For almost every $x$ with respect to $P(x)$, we say that $Z^{\star}$ \textbf{distinguishes} latent equivalent classes $\{O_1,\dots,O_{|O|}\}$, if the matrix  $\Pi (x)=[\pi_j(x,z_l)]_{l,j}, l=1,\dots,e, j=1,\dots,|O|$ has full rank $|O|$. 
\label{df:DE}
\end{definition}

  With imperfect proxies, the second completeness condition fails, and the operator $g \mapsto \mathbb{E}[g(Z)|x,w]$ is not injective. Nevertheless, Theorem \ref{Th:RankCom} shows that the cross-domain rank condition in Definition \ref{df:DE} is  strictly weaker than the first completeness condition over operator $l \mapsto \mathbb{E}[l(U)|x,z]$. With this rank condition, the  point-identification of $\mathbb{E}[Y|X=x,Z=S+1]$ remains possible.

\begin{theorem} [Cross-domain rank condition and completeness]
    The cross-domain rank condition in Definition \ref{df:DE} is strictly weaker than the first completeness assumption in the following sense: it only requires injectivity at the level of proxy-induced latent equivalent classes, whereas the first completeness assumption requires injectivity over full latent space.
    \begin{proof}
        Refer to Appendix \ref{sc:THRankCom} for the proof.
\end{proof}
    \label{Th:RankCom} 
\end{theorem}

In practice, the cross-domain rank condition requires non-redundancy of  the observed source environments. In other words, these environments must capture fundamentally different distributions of the latent confounder by mixing the proxy-induced LECs in sufficiently diverse ways. Although the mixture matrix  $\Pi(x)$ is latent, this requirement is empirically evidenced as a high effective rank of the stacked environment-specific CMEs. 
We empirically demonstrate this relationship between environment diversity, proxy imperfectness, and effective rank in Section \ref{sc:EffImp}. 

As discussed in Section \ref{scc:PCI}, establishing the existence of a bridge is not guaranteed by completeness; it requires specific solvability and regularity assumptions. We formalize the solvability assumptions as follows.

\begin{assumption} [LEC moment condition]
Consider distinguishing environment set $Z^{\star}$ and mixture weight matrix $\Pi (x)=[\pi_j(x,z_l)]_{l,j}, l=1,\dots,e,  j=1,\dots,|O|$ (Definition \ref{df:DE}). Let $r(x,z) = \mathbb{E}[Y \mid X=x,Z=z]$, define $r_x:=(r(x,z_l))_{l=1}^e$. Then, we assume the following condition holds:
\begin{enumerate}
    \item \textbf{Compatibility:} For almost surely x, $r_x \in Range(\Pi
    (x))$. Equivalently, there exists a measurable function $g(x) \in R^{\mid O\mid}$ such that $\Pi(X)g(x)=r_x$.
    \item \textbf{LEC proxy spanning:}
    $p_j \in L_2(\mathcal{W})$ for all $j=1,\dots,|O|$, and the gram matrix $G_{jk}:=\int_W p_j(w)p_k(w)dw \quad j,k=1,\dots,|O|$ is invertible.
    \item \textbf{Square-integrability} The function $g(x)=g_1(x), \dots, g_{\mid O\mid}(x)$ can be chosen such that $g_j\in L_2(\mathcal{X})$ for all $j$.
\end{enumerate}
\label{as:LECMom}
\end{assumption}
\begin{remark} [Role and verifiability of Assumption \ref{as:LECMom}]
     Assumption \ref{as:LECMom} adapts traditional solvability assumptions to the LEC-based framework:
    \begin{itemize}
        \item Compatibility is analogous to Picard/range condition (e.g. Assumption (vii) in \citet{miao2018identifying}; Assumption 2 in Proposition A.2 of \citet{tsai2024proxy}). It requires that outcome-relevant latent variations are distinguishable across different LECs. Compatibility is not directly testable as LECs are latent. However, it can be tested indirectly. If using more informative proxies or environment diversity raises the effective rank without improving target prediction, this suggests the proxy may not capture the outcome-relevant latent shift.
        \item LEC proxy spanning corresponds to operator non-degeneracy and boundedness requirements (e.g. Assumption (v) in \citet{miao2018identifying}; Assumption 2 in Proposition A.2 in \citet{tsai2024proxy}. It ensures that proxy distribution induced by different LECs provides non-redundant direction in proxy space. Although the LECs are latent, their mixtures appear in the observable CMEs. Therefore, the singular values and rank of stacked CMEs provide a direct empirical diagnostic for this condition. We explicitly test this diagnostic and show how proxy imperfection impacts the effective rank in section \ref{sc:EffImp}.
\item Square integrability is a standard regularity condition   ensuring the bridge lies in a stable $L_2$ function class (e.g. Assumption (vi) from \citet{miao2018identifying}; Assumption 3 in Proposition 4.2 of \citet{tsai2024proxy}). It is not directly testable for finite data. However, in practice, the use of bounded kernels and ridge regularization restricts the learning process to a stable class, ensuring this condition holds.
    \end{itemize}
\end{remark}

The  next theorem shows the existence and set-identifiability of the bridge function, which is required for guaranteeing point-identification of $\mathbb{E}[Y|X=x,Z=S+1]$. 
\begin{theorem} [Existence and set identifiability of bridge function]
    Consider imperfect proxies (Definition \ref{df:Imp}) and let $Z^{\star}$ be a distinguishing environment set  (Definition \ref{df:DE}). Also, Lemma \ref{lm:CMEDEC} holds.
     Under Assumption \ref{as:LECMom}, the  following hold:
    \begin{itemize}
        \item There exists at least one bridge function $h$ such that:
          \begin{equation}
          \mathcal{T} _{joint}h=r, \qquad \text{that is,} \quad \mathcal{T} _zh =r(.,z) \quad \forall z \in Z^{\star}
        \end{equation}
        \item The set of all bridge functions solving $\mathcal{T}_{joint}h=r$ is an affine set:
        
         \begin{equation}
    \begin{aligned}
         \mathcal{H}(r)=\{h:\mathcal{T}_zh=r(.,z) (\forall z \in Z^{\star})\} 
        =h_0+\bigcap_{ z \in Z^{\star}} ker(\mathcal{T}_z)
    \end{aligned}
    \end{equation}
    
    where $ker(\mathcal{T}_z)$ is null space of $\mathcal{T}_z$ and $h_0$ is any particular solution. 
    \end{itemize}
\begin{proofsketch}
Under Assumption \ref{as:LECMom} and Lemma \ref{lm:CMEDEC} $p(w \mid x,z)=\sum_{j=1}^{\mid O \mid} \pi_j(x,z)p_j(w)$, the equation $\mathcal{T} _{joint}h=r$ can be written for each $x$ as the linear system $r_x= \Pi(X)g(x)$ where $g_j(x)=\int_{\mathcal{W}}h(x,w)p_j(w) dw$. The compatibility assumption guarantees such $g(x)$ exists, and the LEC proxy spanning assumption allows constructing $h_0\in  L_2(\mathcal{X} \times \mathcal{W})$ with these moments. Therefore, there exists at least one bridge function that solves $\mathcal{T}_{joint}h_0=r$. Due to linearity of $\mathcal{T}_{joint}$, any two solutions $h_1$ and $h_2$ to $\mathcal{T}_{joint}h=r$ satisfy $h_2-h_1 \in ker(\mathcal{T}_{joint})=\bigcap_{ z \in Z^{\star}} ker(\mathcal{T}_z)$. Therefore, $\mathcal{H}(r)= h_0+\bigcap_{ z \in Z^{\star}} ker(\mathcal{T}_z)$ (Refer to Appendix \ref{sc:exbrg} for extended proof).
\end{proofsketch}
    \label{th:exbrg}
\end{theorem}

Finally, based on the rank conditions for LECs and the existence of a bridge function, we prove point-identification of the predictor $\mathbb{E}[Y|X=x,Z=S+1]$.

\begin{theorem} [Point-identification of predictor]
    Assume conditions of Theorem \ref{th:exbrg} and Lemma \ref{lm:CMEDEC} hold.  Consider set of distinguishing environments $Z^{\star}$ (Definition \ref{df:DE}).     Under these conditions, the predictor $\mathbb{E}[Y|X=x,Z=S+1]$  is point-identified.
\begin{proofsketch}
    By Theorem \ref{th:exbrg}, the set of the bridge functions $\mathcal{H}(r)$ is non-empty. For any two solutions $h_1,h_2 \in \mathcal{H}(r)$, let $\Delta h=h_2-h_1$, then  $\mathcal{T}_{joint} \Delta h =0$, implying $\mathcal{T}_{z} \Delta h =0$ for all distinguishing environment set $z \in Z^{\star}$.  By Lemma \ref{lm:CMEDEC}, $\mathcal{T}_z$ decomposes into latent operators $\mathcal{T}_{O_j}$. By full rank condition of distinguishing environment set, $\mathcal{T}_{S+1}$ can  be expressed as linear combination of distinguishing environments operators $\mathcal{T}_{S+1}=\sum_{z\in Z^{\star}}\gamma_z \mathcal{T}_z$. Consequently, $\mathcal{T}_{S+1}\Delta h=\sum_{z\in Z^{\star}}\gamma_z(\mathcal{T}_zh)=0$. This implies $\mathcal{T}_{S+1}h_1=\mathcal{T}_{S+1}h_2$. Therefore, $\mathbb{E}[Y|X =x,Z=S+1]$ is point-identified.
\end{proofsketch}

    \label{th:Point}
    \end{theorem}

\section{Proximal Quasi-Bayesian Active Learning (PQAL)}

Having established that the predictor is point-identified in the  target domain if a  cross-domain rank condition holds, we next present a budget-constrained active learning framework for learning such a predictor.
The goal is to learn the CME operator and adapt the bridge function to the  target domain while considering a limited budget of proxy and label queries from the target.
Figure \ref{fig:OVPBAL} shows an overview of the PQAL framework, including
 \textbf{Initial training:} CMEs and the bridge function are trained by  labeled samples from source environments (lines 2-3 of Algorithm \ref{Al:algorithm}), \textbf{Sample selection:} Most informative samples are selected to query from an external information source (lines 5-15 of Algorithm \ref{Al:algorithm}), and 
     \textbf{Adaptation:}  CMEs in the target domain are trained based on obtained proxies, and source bridge parameters $\alpha_s$ are fine-tuned and updated to  $\alpha$ based on  samples with both proxy and labels (lines 16-17 of Algorithm \ref{Al:algorithm}).
The algorithm returns updated bridge parameters $\alpha$ and CMEs.

\begin{figure}
    \centering
\includegraphics[width=0.5\linewidth]{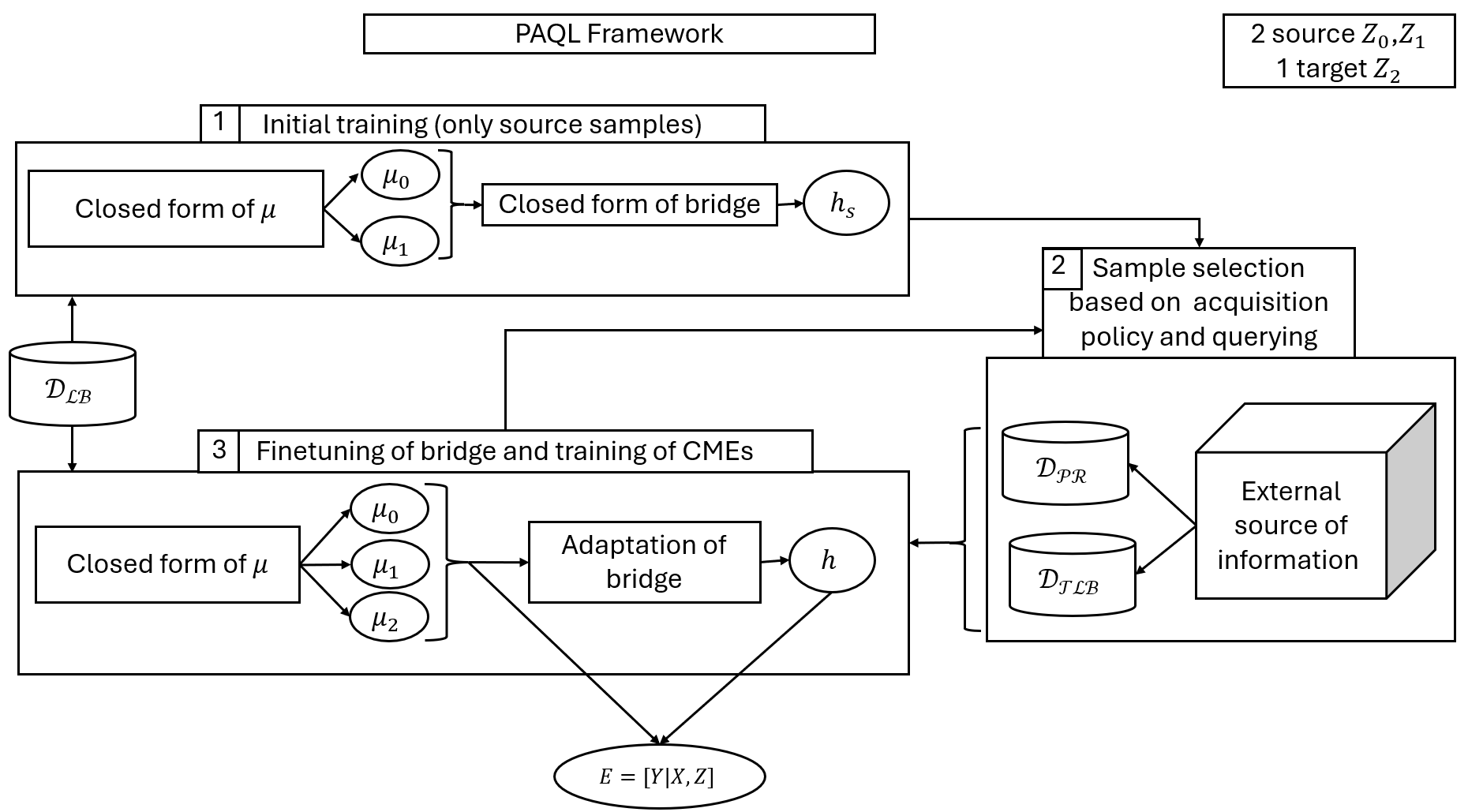}
    \caption{Overview of PQAL framework}
    \label{fig:OVPBAL}
\end{figure}

\subsection{Initial training}

We adapt the KPV-based approach \citep{tsai2024proxy, mastouri2021proximal} to estimate CMEs and bridge functions. KPV trains CMEs  and bridge functions  by solving a two-step ridge regression problem.
In the first step, given a set of labeled samples from the source domains $\mathcal{D}_{\mathcal{LB}}$,
the goal is to solve the  optimization problem
\begin{equation}
   C_{W|X,Z} =\argmin _{C} \frac{1}{m_1} \sum_{i=1}^{m_1} ||\phi(w_i) - C(\phi(x_i)\otimes\phi(z_i))||^2 + \lambda_{CME}||C||^2.
\end{equation}
Since $\mu_{W|X=x,Z=z}=C_{W|X,Z}(\phi(x) \otimes \phi(z))$, the closed-form of CME is calculated as follows \citep{tsai2024proxy,mastouri2021proximal}:
\begin{equation}
\mu_{W|X=x,Z=z}=\sum_{i=1}^{m_1} b_i(x,z) \phi(w_i)
    \label{eq:CME},
\end{equation}
where $b(x,z)=(\mathcal{K}_Z \odot \mathcal{K}_X+\lambda_{CME} m_1 I )^{-1}(\Phi_Z(z) \odot \Phi_X(x))$.
Then, given set of  mutually exclusive labeled samples  $\{(\tilde{x}_j,\tilde{w}_j,\tilde{z}_j,\tilde{y}_j)\}^{m_2}_{j=1}$, the bridge function can be estimated by solving the optimization problem \citep{mastouri2021proximal}
\begin{equation}
    \min_h \frac{1}{m_2} \sum_{j=1}^{m_2} (\tilde{y}_j-<h, \phi(\tilde{x}_j)\otimes \mu_{W|X=\tilde{x}_j,Z=\tilde{z}_j}>)^2+\lambda_{BRG} ||h||_{\mathcal{H}}^2
        \label{eq:brgopt}.
\end{equation}

The optimal closed-form solution for the bridge can be obtained as  \citep{mastouri2021proximal}
\begin{equation}
    \begin{aligned}
\hat{h}_0=\sum_{i=1}^{m_1}\sum_{j=1}^{m_2} \alpha_{i,j} \phi(\tilde{x}_j)\otimes \phi(w_i)  
    \end{aligned}
    \label{eq:brgcl},
\end{equation}
where $
        vec(\alpha) =(I \bar{\otimes} \Gamma) (\lambda_{BRG}m_2I+\Sigma)^{-1} \tilde{Y}$ is the vectorized version of $\alpha$  and $\Sigma= (\Gamma^T \mathcal{K}_W \Gamma) \odot \mathcal{K}_{\tilde{X}}$ and $\Gamma= (\mathcal{K}_Z \odot \mathcal{K}_X + \lambda_{CME} m_1 I) (\mathcal{K}_{Z\tilde{Z}} \odot \mathcal{K}_{X\tilde{X}})$.

\subsection{Acquisition function and sample selection}\label{sc:Ac}

Given a pool of candidate samples without proxies $\mathcal{D}_{\mathcal{PL}}$. The goal is to select a set of samples from $\mathcal{D}_{\mathcal{PL}}$ that decreases predictive uncertainty \citep{smith2023prediction}.  Define design  $d=(x,z)$ and target design $d_*=(x_*,z_*)\sim q(d_*)$. In our setting, EPIG criterion becomes
\begin{equation}
    EPIG(d)=\mathbb{E}_{q(d_{*})}[ I(y;y_*|d,d_*)].
\end{equation}
By considering  that both CME and bridge are trained by ridge regression, which  corresponds to a Gaussian  process \citep{kanagawa2018gaussian}, we conclude that the  outputs  also follow a joint Gaussian distribution. As a result, $[I(y;y_*|d,d_*)]$ has a closed form $  I(y;y_*|d,d_*) = -\frac{1}{2} \log (1-\rho^2(d,d_*)) 
$ obtained based on  conditional entropies, where  $\rho(d,d_*)=\frac{Cov(y,y_*|d,d_*)}{\sqrt{\mathbb{V}(y|d)\mathbb{V}(y_*|d_*)}}$ (see Appendix \ref{sc:DerAct}).
Although based on this equation
  a closed-form for EPIG based on predictive variances ($\mathbb{V}$) and covariances ($Cov$) is achievable, evaluation of this equation is intractable in our setting due to the following reasons: 
 \begin{itemize}
     \item The bridge function is maintained as a point estimate. As a result, the posterior uncertainty of the bridge and outcome level prediction variance are unavailable.
     \item At query time, the proxy $W$ is  unobserved.
 \end{itemize}
 However, we  use the CME posterior uncertainty  as surrogate for EPIG, since the outcome is a linear functional of  CME: $\hat{y}=T(d)(\mu_{W|x,z})$  where $T(d)= <h_{adp}, \phi(x) \otimes (.)>$. 
 Since CMEs for each environment are trained on samples from that environment, we first filter samples by environment id $ X^{r}=\{x^{(r)}_j:z^{(r)}_j=z\}$ , $ Z^{r}=\{z^{(r)}_j:z^{(r)}_j=z\}$. The  CME posterior uncertainty $\sigma_{\mu}$  \citep{chowdhury2020active} is 
\begin{equation}
 \begin{aligned}
        &\sigma^2_{\mu,r}(d)=  \sigma^2_{\mu,r}(x,z)= \\
        &k(x,x)  k(z,z) - (\Phi_{X^{(r-1)}}(x) \odot \Phi_{Z^{(r-1)}}(z))^{\top} (\mathcal{K}_{Z^{(r-1)}} \odot \mathcal{K}_{X^{(r-1)}}+n_{z,{r-1}} \lambda_{CME} I)^{-1} (\Phi_{X^{(r-1)}}(x) \odot \Phi_{Z^{(r-1)}}(z))
 \end{aligned}
     \label{eq:CMEUnc},
 \end{equation}

where $n_{z,{r-1}}$ is the number of current queried samples at round $r-1$ with environment id $z$.  Our acquisition selects samples from environments with large  posterior uncertainty of the CME. In Appendix \ref{sc:GarAcq}, we formally show that PQAL with CME-based uncertainty acquisition  can  recover the 
maximal informative rank from candidate pool.

\subsection{Bridge and CME functions adaptation}

Consider a set of initial  labeled samples from source domains ($\mathcal{D}_{\mathcal{LB}}$), and a set of queried samples with proxy  at  round $r$ ($\mathcal{D}_{\mathcal{PR}}^{r}$) from  source and  target environments. 
 We leverage a small set of labeled target samples 
($\mathcal{D}_{\mathcal{TLB}}^{r}$)
to stabilize bridge adaptation in the finite-sample setting.

\textbf{Updating CMEs:}  As CMEs for each environment are trained on samples from that environment, we first filter samples by environment id $ \mathcal{D}_{\mathcal{PR}}^{r}(z)=\{(x^{(r)}_j,w^{(r)}_j):z^{(r)}_j=z\}$ . We train CMEs for both  source and target environments by adapting Equation \ref{eq:CME} as
\begin{equation}
 \mu^{(r)}_{W|X=x,Z=z}=\sum_{j \in \{j:  z^{(r)}_j=z\}} b_j^{(r)}(x,z) \phi(w^{(r)}_j)
    \label{eq:CMEAdp},
\end{equation}
where $b^{(r)}(x,z)=(\mathcal{K}_{Z^{(r)}} \odot \mathcal{K}_{X^{(r)}}+\lambda_{CME} n_{r,z} I)^{-1} (\Phi_{Z^{(r)}}(z) \odot \Phi_{X^{(r)}}(x))$ and $ |\mathcal{D}_{\mathcal{PR}}^{r}(z)|=n_{r,z}$. Also, $Z^{(r)}$ and $X^{(r)}$ are samples of $X$ and $Z$ in $\mathcal{D}_{PR}$ in iteration $r$.

\textbf{Adaptation of bridge:} After initial training of bridge function $\hat{h}_0$  on source domains (Equation \ref{eq:brgcl}), we obtain optimal parameters $\alpha_s$. We fine-tune this bridge function in the target domain using  samples queried during training.
$\hat{h}_0$ has a stable representation due to the abundance of  covariate and proxy pairs in the source domains. However, during the active learning loop,  only a few target samples are available, which risks overfitting and destabilizing the optimization. To prevent these issues, we represent target samples in the fixed feature space derived from the source domains.
As a result, the bridge is adapted to the target domain while maintaining the stability established during source training. The predicted outcome $f_{\alpha}$ is defined over fixed source basis $\{\phi(w_i)\}_{i=1}^{m_1}$ and $\{\phi(\tilde{x_j})\}_{j=1}^{m_2}$ as
\begin{equation}
    f_{\alpha}(x,w)=\sum_{i=1}^{m_1}\sum_{j=1}^{m_2} \alpha_{i,j} k_{\mathcal{W}}(w_i,w) . k_{\mathcal{X}}(\tilde{x}_j,x),
\end{equation}
where $\alpha \in  \mathbb{R}^{m_1 \times m_2}$ are  the bridge parameters. At the start of adaptation, we initialize $\alpha^{(0)} := \alpha_s$.
Then, we adapt a bridge by solving the regularized optimization problem

\begin{equation}
\begin{aligned}
     & \hat{\alpha}^{(r)}= \argmin_{\alpha} (\mathcal{L}_{source}(\alpha)+ \lambda_{tgt} \mathcal{L}^{(r)}_{target} (\alpha) +  
    \lambda_{sim} \mathcal{L}^{(r)}_{manifold} (\alpha) + \lambda_{reg}\mid\mid \alpha -\alpha^{(0)}\mid \mid_2^2).  \\
&\mathcal{L}_{source}(\alpha)= \frac{1}{\mid \mathcal{D}_{\mathcal{LB}}\mid} \sum_{i=1}^{\mid \mathcal{D}_{\mathcal{LB}} \mid} (y_i-f_{\alpha}(x_i,w_i)) ^2. \\
&\mathcal{L}^{(r)}_{target} (\alpha) = \frac{1}{\mid \mathcal{D}_{\mathcal{TLB}}^{r} \mid}  \sum_{j=1}^{\mid \mathcal{D}^{r}_{\mathcal{TLB}} \mid} (y^{(r)}_j-f_{\alpha}(x^{(r)}_j,w^{(r)}_j)) ^2.
\\
&\mathcal{L}^{(r)}_{manifold} (\alpha) = \sum_{j=1}^{\mid \mathcal{D}^{r}_{\mathcal{TLB}}\mid} \sum_{k=1}^{\mid \mathcal{D}_{\mathcal{PR}}^{r} \mid} \mathcal{S}_{jk} (f_{\alpha}(x^{(r)}_j,w^{(r)}_j) - f_{\alpha}(x^{(r)}_k,w^{(r)}_k))^2.
\end{aligned}
             \label{eq:losses}
\end{equation}
In Equation \ref{eq:losses}, the  first component  $\mathcal{L}_{source}(\alpha)$  evaluates how well the current  bridge, parameterized by  $\alpha$, fits the source data. This loss prevents catastrophic forgetting \citep{kemker2018measuring}. The second component $\mathcal{L}^{(r)}_{target} (\alpha)$   minimizes the prediction error on a few queried target samples with labels. The third component $ \mathcal{L}^{(r)}_{manifold} (\alpha)$ 
encourages smooth prediction over the target manifold  via the similarity kernel $\mathcal{S}$ between generated outcomes from   labeled and unlabeled target samples. We set  $\mathcal{S}$ to a Gaussian RBF kernel on the joint space  $\mathcal{X} \times \mathcal{W}$. 
At prediction time, the final output for $(x_{new},z_{new})$ is $ \hat{y}_{new}= <h_{adp}, \phi(x_{new}) \otimes \mu_{W|x_{new},z_{new}}>$ and $h_{adp}=\sum_{i=1}^{n_1}\sum_{j=1}^{n_2} \alpha_{i,j}  \phi(\tilde{x}_j) \otimes\phi(w_i)$.

\section{Experiments}

We evaluate our approach on synthetic regression tasks consistent with the causal DAG and assumptions in Section \ref{sc:Prob}. After outlining the data generation process and baselines, we study the effect of imperfect proxies on  identifiability, test the robustness of our method to varying degrees of distribution shifts, and evaluate the  effectiveness of the proposed acquisition function. For  experiment setup please refer to Appendix \ref{sc:setup}.

\subsection{Data generation process} \label{sec:data}

Verifying the accuracy of our domain adaptation method requires known structural ground truth that complies with latent shift assumptions. To circumvent this, standard approach is to use semi-synthetic benchmarks \citep{louizos2017causal}. By blending real-world covariates with synthetic confounding mechanisms, we can establish a reliable baseline to  accurately measure performance.  Following this methodology, we consider a diverse set of synthetic and semi-synthetic data generation processes, ranging from controlled regression tasks and the dSprites image dataset \citep{tsai2024proxy} to complex, real-world tabular benchmarks, such as  IHDP \citep{ruth2024infant} and ACS Folktables \citep{ding2021retiring}. Crucially, across all evaluated datasets, we explicitly design the generation process for $W$ to exhibit the characteristics of imperfect proxies.

\textbf{Dataset 1 (Continuous proxy):}
We sample $U$ from a beta distribution with parameters $a$ and $b$.  We generate a continuous proxy  $W$  using a sinusoidal mapping with Gaussian noise $\epsilon$, where $B$ controls the degree of non-injectivity of the  mapping from $U$ to $W$. Equation \ref{eq:datset1} shows a structural causal model (SCM) for Dataset 1.

\textbf{Dataset 2 (Discrete proxy):}
We also consider the case of a discrete proxy, as some  approaches, such as  \citet{prashant2025scalable}, rely on the discrete proxy assumption. In addition, we consider a nonlinear mapping from $U$ and $X$ to $Y$. Equation \ref{eq:Dataset2} shows a SCM for Dataset 2.
\begin{minipage}{0.48\textwidth}
\centering
    \begin{equation}
    \begin{aligned}
U \sim Beta(a,b)\\
  W= \sin(2\pi B U)+\epsilon  \quad \epsilon \sim \mathcal{N}(0,\sigma_w^2) 
   \\ X \sim \mathcal{N}(0,1)\\
        Y = (2U - 1) X 
    \end{aligned}
    \label{eq:datset1}
\end{equation}
\end{minipage}
\hfill
\begin{minipage}{0.48\textwidth}
   \centering
\begin{equation}
    \begin{aligned}
&U \sim Beta(a,b)\\
 & W= 
  \begin{cases}
      \lfloor U.B \rfloor & \text{with probability} \; 1- \eta \\
      R &  \text{with probability} \;  \eta
  \end{cases} 
   \\& X \sim \mathcal{N}(0,1)\\
      &  Y = U^3 X 
    \end{aligned}
    \label{eq:Dataset2}
\end{equation}
\end{minipage}

Where $R$ is sampled uniformly from bins ($R \sim Unif \{0,1,\dots, B-1\}$) and $B$ is the number of bins to control  imperfectness in the  proxy. $\eta$ is the probability of proxy corruption. When the proxy is corrupted, W is reassigned uniformly at random to one of the $B$ bins.

\textbf{Dataset 3 (dSprites):}
To demonstrate PQAL's ability in handling more complex datasets, we use the semi-synthetic dSprites dataset \citep{dsprites17}. Data is generated according to the following SCM, where the latent confounder $U$ is the object orientation. Function $f_{image}$ generates an image from latent  generative factors. For our experiments, we vary only orientation as a confounder, while holding all other generative factors (e.g., shape and scale) fixed. Figure \ref{fig:dspvis} shows the effect of shifts in the distribution of $U$ in the generated image and the data is created as
\begin{equation}
    \begin{aligned}
U \sim Beta(a,b)\\
  W= \sin(2\pi B U)+\epsilon  \quad \epsilon \sim \mathcal{N}(0,\sigma_w^2) 
   \\ X \sim f_{image}(u)+\epsilon_x\\
        Y = (2U - 1) X 
    \end{aligned}
\end{equation}

\textbf{Dataset 4 (IHDP Benchmark):}
We adapt the Infant Health and Development Program (IHDP) dataset \citep{ruth2024infant}, which is built from a real-world clinical trial on premature infants. 
We partition the data by  birth weight to construct distinct source and target environments ($Z$).  We retain the actual clinical covariates ($X$) and environment assignments ($Z$) to preserve realistic marginal distributions. To establish an accurate baseline,  we synthetically generate the unobserved confounder $U$, the outcome $Y$, and the imperfect proxy $W$. Intuitively, $U$  represents unmeasured clinical risks, such as parental stress or unrecorded genetic predispositions. The outcome $Y$ can be viewed as a clinical severity score indicating patient health status. Finally, $W$ can be interpreted as noisy risk indicator such as parental risk score.
\begin{equation}
    \begin{aligned}
        U \sim Beta(a,b) \\
         W= Scale(\sin(2\pi B U)+\epsilon_w ) \quad \epsilon_w \sim \mathcal{N}(0,\sigma_w^2) \\
         Y=(2U-1)X_{proj}+\epsilon_y \quad \epsilon_y \sim \mathcal{N}(0,\sigma_y^2)
    \end{aligned}
\end{equation}
where Scale(.) denotes standardization to zero mean and unit variance, $X_{proj}$ is the mean projection of observed covariates.
\\
\textbf{Dataset 5 (ACS Folktables):}
We utilize the Folktables \citep{ding2021retiring}
dataset that provides records containing highly entangled socioeconomic 
features such as education, age, and occupation across different states of the United States. We retain the actual demographic covariates $X$ and geographic environment assignment ($Z$) to preserve realistic marginal distributions. Because true confounders are inherently latent in this dataset, we establish an accurate baseline by synthetically generating the unobserved regional confounder $U$, the imperfect proxy $W$, and the outcome $y$. Intuitively, $U$ is an unmeasured regional socioeconomic pressures, such as the localized cost of living. Also, $Y$ can be viewed as an  economic mobility score that summarizes overall financial success of individuals.  Finally, $W$ can be interpreted as housing affordability score.
\begin{equation}
    \begin{aligned}
        U \sim Beta(a,b) \\
         W= Scale(\cos(2\pi B U)+\epsilon_w ) \quad \epsilon_w \sim \mathcal{N}(0,\sigma_w^2) \\
         Y=(2U-1)X_{proj}+\epsilon_y \quad \epsilon_y \sim \mathcal{N}(0,\sigma_y^2)
    \end{aligned}
\end{equation}

\subsection{Effect of imperfect proxies on  LECs} \label{sc:EffImp}
Here, we show how  proxy imperfection affects the  required diversity for environments  to point-identify  the predictor $\mathbb{E}[Y|X=x,Z=S+1]$, serving as the empirical diagnostic for the cross-domain rank condition and LEC proxy spanning (Assumption \ref{as:LECMom}) discussed in Section \ref{sc:Theory}. 
As explained in section \ref{sc:Theory}, the conditional mean embedding  for  a fixed $x$ can be expressed as mixture over LECs: $ \mu_{x,z} = \sum_{j=1}^{K}\pi_j(x,z)m_j$ where  $m_j := \mathbb{E}[\phi(W)|U \in O_j]$. Since $m_j$s are invariant across environments, any observable variation between the conditional mean embeddings is based on mixing weights $\Pi(x)$. Therefore, identifying linear independence (environment variation) in mixture weights is mathematically equivalent to the evaluation of the  singular values  of a  matrix that is formed by stacking CMEs of different environments. We report these values as the effective rank.

We use Dataset 2 for this experiment as discrete proxy bins provide a direct control over proxy non-injectivity. Figure \ref{fig:RANVSBins} shows that as the number of bins increases (fine-grained proxies), the effective rank increases, suggesting environments become more distinguishable in CME space.  In the limit with an infinite number of samples, this rank converges to the rank of the population stacked CME matrix (upper-bounded by the number of total environments; e.g. 3 for the case of 2 sources and 1 target).

\begin{figure}
    \centering
\begin{minipage}[t]{0.47 \textwidth}
\vspace{0pt}
\centering
\includegraphics[width=\linewidth]{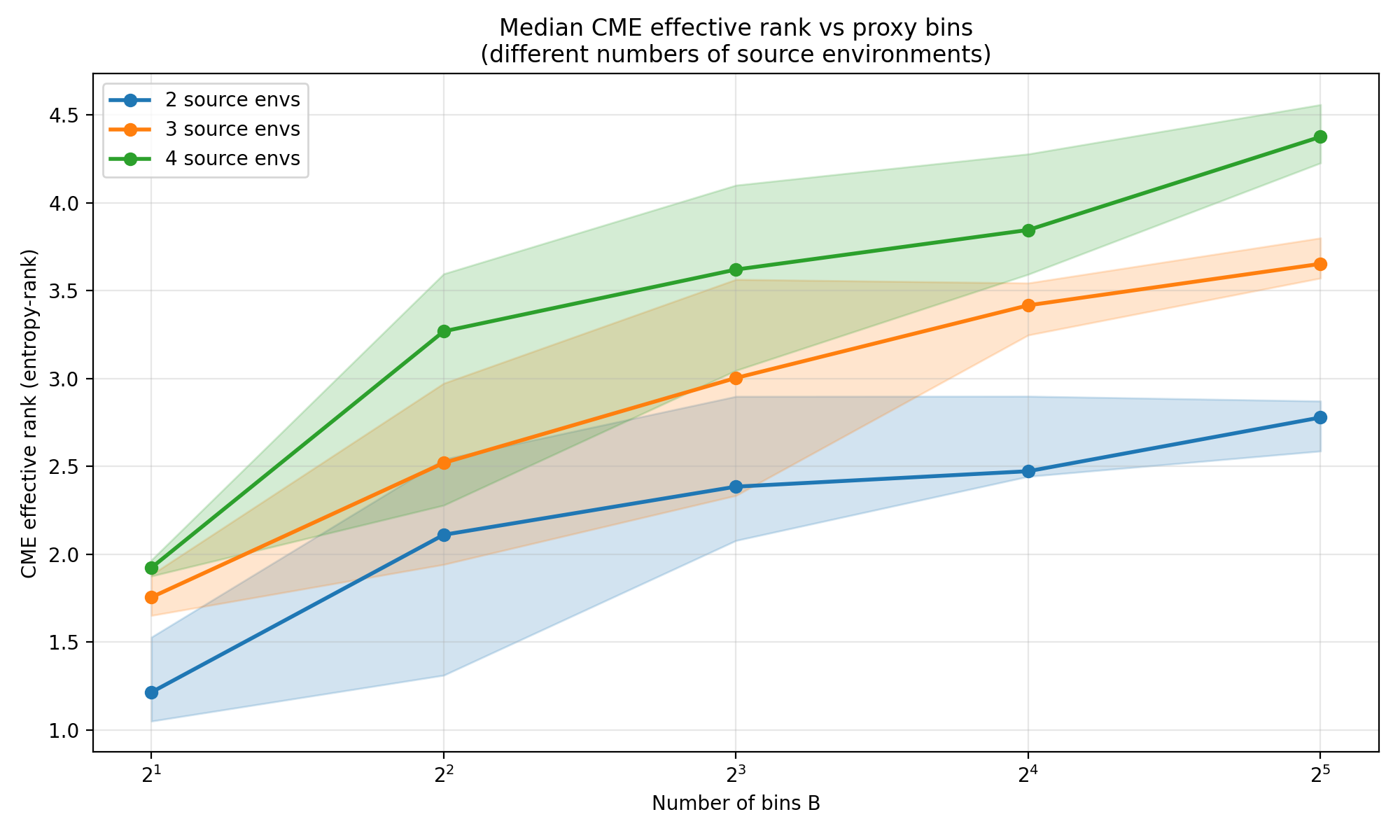}    \captionof{figure}{Effective rank with respect to proxy imperfectness for different number of environments} \label{fig:RANVSBins}
\end{minipage}
    \begin{minipage}[t]{0.47 \textwidth}
    \vspace{-8 mm}
       \centering
    \captionsetup{width=\linewidth}
 \captionof{table}{
 MSE across all datasets. Blue and bold mark the lowest and second-lowest MSEs.}
 \label{tab:D1D2tables}
 \resizebox{\linewidth}{!}{\begin{tabular}{|l|r|r|r|r|r|}

\hline
\multicolumn{6}{|c|}{\textbf{Dataset1}}\\
\hline
Strategy & Degree 1 & Degree 2 & Degree 3  & Degree 4 & Degree 5 \\
\hline

FewShot\_ERM & \textbf{0.0729} & 0.1708 & 0.3061 & 0.5229 & 0.7278 \\
\hline
Proxy\_DA & 0.1083 & 0.2902 & 0.4194 & 0.8638 & 1.0926 \\
\hline

Oracle & \textcolor{blue}{0.0473} & \textcolor{blue}{0.0814} & \textcolor{blue}{0.1093} & \textcolor{blue}{0.1721} & \textcolor{blue}{0.1974} \\
\hline
\Xhline{1pt}
PQAL & 0.0800 & \textbf{0.1388} & \textbf{0.2025} & \textbf{0.3293} & \textbf{0.4015} \\
\hline
\Xhline{1pt}
\multicolumn{6}{|c|}{\textbf{Dataset2}}\\
\hline
FewShot\_ERM & 0.0071 & 0.0212 & 0.0576 & 0.1594 & 0.2787 \\
\hline
Proxy\_DA & 0.0109 & 0.0240 & 0.0732 & 0.1934 & 0.4112 \\
\hline
SC& 0.0098 & 0.0380 & 0.1013 & 0.2547 & 0.4727 \\
\hline
Oracle &  \textcolor{blue}{0.0037} &  \textcolor{blue}{0.0144} &  \textcolor{blue}{0.0339} &  \textcolor{blue}{0.0784} &  \textcolor{blue}{0.1101} \\
\Xhline{1pt}
PQAL & \textbf{0.0053} & \textbf{0.0183} & \textbf{0.0466} & \textbf{0.1201} & \textbf{0.2216} \\
\hline
\Xhline{1pt}
\multicolumn{6}{|c|}{\textbf{dSprites}}\\
\hline

FewShot\_ERM & $0.186 9$ & $0.604 $ & $0.969 $ & $1.199 $ & $1.256$ \\
\hline

Proxy\_DA & \textbf{0.054} & $0.334 $ & $0.708 $ & $1.004 $ & $1.071$\\
\hline

Oracle & \textcolor{blue}{$0.050 $} & \textcolor{blue}{$0.160 $} & \textcolor{blue}{$0.310 $} & \textcolor{blue}{$0.427 $} & \textcolor{blue}{0.449} \\
\hline

PQAL & $0.055 $ & \textbf{0.303} & \textbf{0.641}  & \textbf{0.888} & \textbf{0.946}\\
\hline

\multicolumn{6}{|c|}{\textbf{IHDP}}\\
\hline

FewShot\_ERM & $0.057$ & $0.081 $ & $0.094 $ & $0.091 $ & $0.086$ \\
\hline

Proxy\_DA &  \textbf{0.034} & $0.082 $ & $0.1 $ & $0.104$ & $0.110$\\
\hline

Oracle & \textcolor{blue}{$0.032$} & \textcolor{blue}{$0.053$} & \textcolor{blue}{$0.064 $} & \textcolor{blue}{$0.061 $} & \textcolor{blue}{$0.066$} \\
\hline

PQAL & $0.039$ & \textbf{0.059} & \textbf{0.068}  & \textbf{0.067} & \textbf{0.067}\\
\hline

\multicolumn{6}{|c|}{\textbf{Folktables}}\\
\hline

FewShot\_ERM & $0.038$ & $0.068 $ & $0.101 $ & $0.137 $ & $0.158$ \\
\hline

Proxy\_DA &  $0.034 $ & $0.062 $ & $0.089 $ & $0.121$ & $0.150$\\
\hline

Oracle & \textcolor{blue}{$0.031$} & \textcolor{blue}{$0.057$} & \textcolor{blue}{$0.078 $} & \textcolor{blue}{$0.104 $} & \textcolor{blue}{$0.128$} \\
\hline

PQAL & \textbf{0.032} & \textbf{0.058} & \textbf{0.082}  & \textbf{0.109} & \textbf{0.132}\\
\hline

\end{tabular}}
    \end{minipage}
    
\end{figure}

\subsection{Robustness under varying degrees of distribution shift}

We examine the robustness of PQAL to varying degrees of shift alongside  four closely related approaches. Proxy-DA \citep{tsai2024proxy} is the closest approach to our study, and identifies the predictor $\mathbb{E}[Y|X=x,Z=S+1]$ but uses the same bridge in both  the source and target environments and  trains CMEs in the target environment  based on given proxies.
\citet{prashant2025scalable} (SC) train  a mixture-of-experts (MOE) model to address latent distribution shift in discrete  confounders.  To adjust for the effects of  confounders, the expert models are reweighted in the target using the recovered distribution of   latent confounders inferred from proxies. Furthermore, we compare our approach  to few-shot-ERM, which   minimizes risk using a few labels from the  target sample, to show the necessity of  adjusting for the effects of latent confounders. Finally, we also considered  Oracle, which has access to all samples in the target pool, providing a lower bound for the errors. For a detailed explanation of baselines, please refer to Appendix \ref{sc:Base}.

 For Dataset 1, 3, 4, 5, we compare all methods except  SC, which requires a discrete proxy. 
 All baseline methods are evaluated using a random acquisition strategy, as they either do not use the CME in their structure or are not designed for  active learning. PQAL actively selects samples based on CME uncertainty. To quantify varying degrees of distribution shift for Datasets 1-3, 
we keep the mean of the Beta distribution approximately fixed and vary its variance by selecting different parameters for the distribution. The variance of the Beta distribution is given by $\frac{ab}{(a+b+1)(a+b)^2}$. We use the resulting variances 
as discrete levels of distribution shift, with larger variance corresponding to stronger latent distribution shift. For  IHDP and Folktables datasets (Datasets 4 and 5), we create a shift by progressively deviating 
the mean of the beta distribution. By monotonically adjusting the shape parameters $(\alpha,\beta)$,  we shift the expected value of the unobserved  confounder toward the extreme tail of the distribution. Therefore, PQAL performance is also evaluated under different directional shifts in the latent confounder mean.
(Please refer to Appendix \ref{sc:setup} for details on the degree of shift and experiment setups).

Table \ref{tab:D1D2tables} summarizes the mean square error (MSE) across these shift levels for all datasets. For the smallest degree of shift, Few-Shot ERM and Proxy-DA outperform our approach for dataset 1, IHDP and dSprite, respectively. However, for other shift degrees PQAL  achieves  lowest MSE. As the degree of shift increases, PQAL exhibits substantially better robustness than  other methods, highlighting the importance of actively adapting to latent confounder shifts.

\subsection{Effect of acquisition function}

We evaluate the proposed acquisition function (CME uncertainty) against   expected error reduction (EER)-inspired \citep{cohn1994active}  surrogate that selects representative $x$ (RPX). In addition, we compare our method with a random and an uncertainty-based acquisition function for $Z$ (refer to Appendix \ref{sc:AcStr} for a detailed explanation of each acquisition strategy). As discussed in section \ref{sc:Ac}, defining any acquisition function based on the uncertainty of the bridge is intractable, and as a result, we do not include any bridge-based acquisition strategies in our experiments.

Figure \ref{fig:DataAcq} shows the MSE of PQAL using different acquisition functions for both synthetic datasets. The shaded region indicates the 95\% confidence interval across random seeds. The CME uncertainty has the lowest MSE and the lowest variability compared to all baselines.
Random acquisition improves with obtaining more queries,  but remains  worse than CME uncertainty in the continuous case. $Z$ uncertainty outperforms random acquisition by encouraging the model to cover  environments. However, it performs worse than  CME uncertainty as it does not account for the effect of covariates when selecting  candidates to obtain a proxy.
EER initially selects covariate representative points that reinforce stable, biased patterns. This selection leads to an increase in MSE as the model overfits  the misspecified bridge function, which is blind to latent shifts.

\begin{figure}
     \centering
     \begin{subfigure}[b]{0.48\textwidth}
         \centering
         \includegraphics[width=\textwidth]{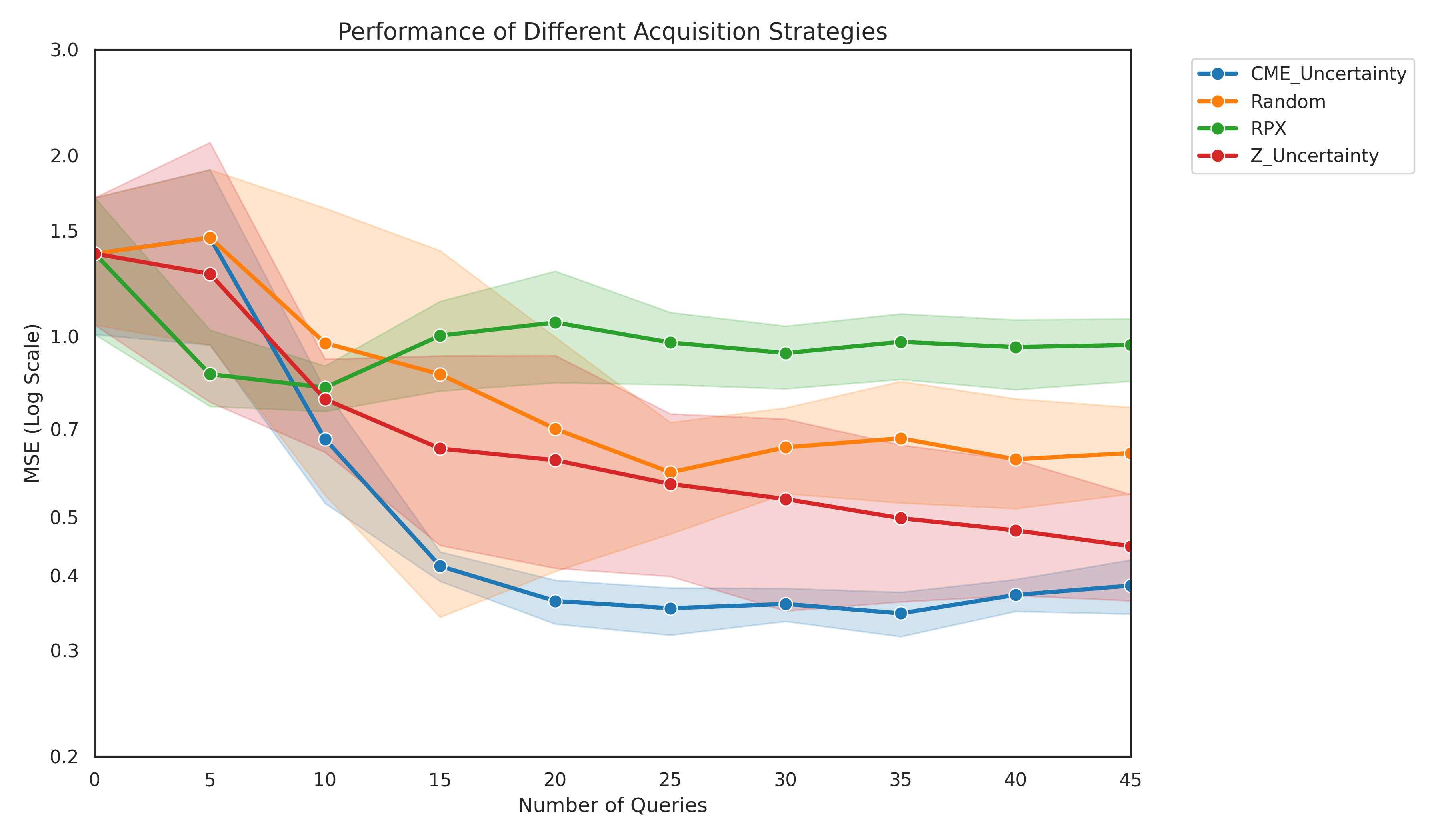}
    
     \end{subfigure}
     \hspace{0.02 \textwidth}
     \begin{subfigure}[b]{0.39\textwidth}
         \centering
         \includegraphics[width=\textwidth]{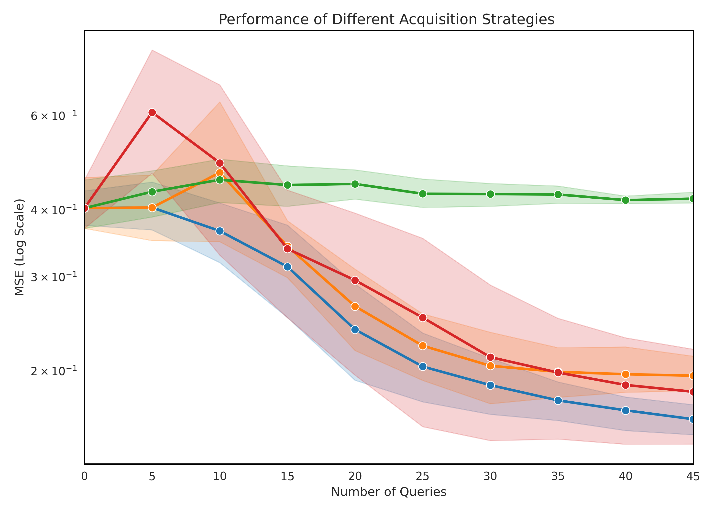}

     \end{subfigure}
        \caption{MSE error for different acquisition functions on Dataset 1(left) and Dataset 2 (right)}
        \label{fig:DataAcq}
\end{figure}

\section{Related work}

\paragraph{Robust predictor under distribution shift:}
Several studies leverage proxies to  train a robust predictor under latent distribution shifts.
\citet{alabdulmohsin2023adapting} use concept and proxy variables to estimate latent subgroup density ratios,  reweight conditional subgroup probabilities, and  recover the  predictor in the target environment without labels. 
\citet{tsai2024proxy} use  PCI to adapt a trained model from the source to the target environment. They  use the same bridge function in both source and target domains and update CMEs using target proxies.
\citet{prashant2025scalable}
 recover the conditional latent confounder distributions based on proxies under a discrete latent confounder shift. They train a mixture-of-experts (MOE) model \citep{xu1994alternative}  for  each value of  latent confounders. To adapt the weights of experts in MOE for the target domain, they used recovered conditional latent confounder distributions.   All aforementioned approaches assume some  form of completeness to establish point-identification guarantees. In contrast, we establish point-identification guarantees for the predictor when completeness fails due to imperfect proxies,  and  recovering latent subgroups and latent variables is not possible.  In addition, rather than  assuming  the diversity condition holds for environments, our framework actively seeks to satisfy it by querying the most informative samples. Furthermore, similar to \citet{tsai2024proxy}, our approach applies to both discrete and continuous proxies and latent confounders.
\paragraph{Causal identification in proxy-based causal inference:}
In causal inference, the use of a proxy to adjust for the effect of unobserved confounders is well established. These approaches often rely on completeness for continuous and invertibility of $P(W\mid U)$ in discrete cases \citep{miao2018identifying}. Recent studies relax completeness at the expense of partial identification (set identification).  
\citet{ghassami2023partial} obtain sharp bounds  on the average treatment effect (ATE) by enforcing dependencies  between proxies, treatment, and outcome. \citet{zhang2022partial} 
bound average causal effect (ACE) by assuming $P(W|U)$ is partially observable and bounded.  While these approaches break completeness, their focus is on the set identification of causal effects such as ATE and ACE. Besides, a source of completeness violation, which is the effect of imperfect proxies, is not formalized in these studies.

\section{Discussion}
While our framework successfully relaxes the completeness assumption to provide point-identification guarantees under imperfect proxies, it relies on specific structural conditions. We outline the failure cases where LEC-based identification becomes too weak:
\begin{itemize}
    \item \textbf{Violation of Remark \ref{RM:NoNLEC} (Degenerate LEC)}:
    If the proxy is entirely uninformative about the latent confounder $U$, it fails to partition the latent space into at least two distinct classes ($|O|=1$). In this degenerate case, the proxy provides no measurable information regarding  domain-induced changes in $U$, making the domain adaptation problem unsolvable.
    \item \textbf{Violation of Definition \ref{df:DE} (Rank-Deficient Environment):}
    Even if the proxy successfully partitions the latent space ($|O| \geq 2$), LEC-based identification remains too weak if the observed source domains are homogeneous. If the environment does not mix the proxy-induced LECs in linearly independent ways, the cross-domain matrix $\Pi(X)$ becomes rank-deficient. Consequently, the target predictor falls back to being set-identified rather than point-identified.
    \item \textbf{Violation of Assumption \ref{as:LECMom} (Incompatibility)}:
    If outcome-relevant variations in the latent confounder fell entirely within a single LEC and thus cannot be distinguished by proxy, the target regression vector $r_x$ will not lie in the range of $\Pi(X)g(x)$. In this scenario, the LEC-level inverse problem is fundamentally unsolvable, and no valid bridge function exists.
    \item \textbf{Violation of Assumption \ref{ass: shift} (Loss of invariance):}
    Our identification relies on the invariant causal mechanisms $P(Y\mid U,X)$ and $P(W \mid U)$ across domains. If the target domain introduces a structural shift in the physical data generation process (e.g., $P_{source}(Y \mid U,X) \neq P_{target}(Y \mid U,X)$), the bridge function learned in the source domain becomes misspecified.
\end{itemize}

\section{Conclusion}
\enlargethispage{1\baselineskip}
This paper addresses domain adaptation  under a latent shift with imperfect proxies. By introducing LECs to capture proxy-induced indistinguishability,  we showed that point-identification of robust predictors is possible through cross-domain rank condition. We  proposed PQAL to  enforce this rank condition. Our experiments supported our claims in the successful adaptation of the predictor to the target environment.

\section*{Acknowledgments}
This work was supported by the Research Council of Finland (Flagship preprogram: Finnish Center for
Artificial Intelligence FCAI and decision 3597/31/2023), ELLIS Finland, EU Horizon 2020 (European Net-
work of AI Excellence Centers ELISE, grant agreement 951847; ERC ODD-ML 101201120; 3597/31/2023,
grant agreement 3597/31/2023), UKRI Turing AI World-Leading Researcher Fellowship (EP/W002973/1).
We also acknowledge the computational resources provided by the Aalto Science-IT Project from Computer
Science IT.

\bibliography{main}
\bibliographystyle{tmlr}

\appendix

\renewcommand{\thefigure}{A\arabic{figure}}
\setcounter{figure}{0}

\section{Proof of theorems}
In this section we provides extended proof for  Lemma \ref{lm:CMEDEC} and Theorems .

\subsection{Proof of Lemma \ref{lm:CMEDEC}}
\label{sc:lmCMEDEC}

\begin{proof}
   Let $(\mathcal{A},\mathcal{W})$  be any measurable set where $A \in \mathcal{A}$. By law of expected iteration \citep{durrett2019probability}, we have:

   \begin{equation}
       \begin{aligned}
           P(W\in A|X=x,Z=z) =\mathbb{E}[1\{W\in A\} |X=x,Z=z] =\\ \mathbb{E}[\mathbb{E}[1\{W\in A\} |U,X=x,Z=z]\bigg{|}X=x,Z=z]=\\
          \mathbb{E}[P(W\in A|U, X=x,Z=z)\bigg{|}X=x,Z=z]
       \end{aligned}
       \label{eq:TE}
   \end{equation}

Based on causal DAG \ref{fig:CAD}, $W \perp (X,Z)  \mid U$, hence: 

\begin{equation}
    P(W \in A|U, X=x,Z=z) = P(W\in A|U)
    \label{eq:Indep}
\end{equation}

By replacing Equation \ref{eq:Indep} in Equation \ref{eq:TE}, we have
\begin{equation}
    P(W\in A|X=x,Z=z) = \mathbb{E}[P(W\in A|U)\bigg{|}X=x,Z=z]
\end{equation}

Based on  regularity assumption (Assumption \ref{ass:Reg}),  LECs $\{O_j\}_{j=1}^{\mid O \mid}$ form a measurable partition of the support of $U$ (Definition \ref{df:LECS}). Applying the law of total probability over this partition leads to:

\begin{equation}
     P(W\in A|X=x,Z=z) =  \mathbb{E}[P(W\in A|U)\bigg{|}X=x,Z=z]=  \sum_{j=1}^{|O|}P(W \in A \mid U \in O_j) P(U \in O_j|X=x,Z=z)
\end{equation}

Define $P_j(A):= P(W \in A| U \in O_j)$ and $\pi_j(x,z):=P(U \in O_j|X=x,Z=z)$, we have:

 \begin{equation}
     P(W \in A|X=x,Z=z)= \sum_{j=1}^{|O|} \pi_j(x,z) P_j(
     A
     )
 \end{equation}

\end{proof}

\subsection{Proof of Theorem \ref{Th:RankCom}}
\label{sc:THRankCom}

    \begin{proof}
   We show that the rank condition in Definition \ref{df:DE} is an LEC-level relaxation of the first completeness assumption. The first completeness assumption requires injectivity over all square-integrable functions of $U$, whereas the rank condition only requires injectivity over functions that are constant within each LEC.

    Let $R \subset L^2(U)$ be the finite-dimensional subspace of functions that are constant on each LEC.
    
    \begin{equation}
        R= span\{1 \{u \in O_j\}: j=1, \dots , \mid O \mid \}.
    \end{equation}
    
     Any $l \in R$ has the form:
     
     \begin{equation}
         l(u)=\sum_{j=1}^{\mid O\mid } \alpha_j 1 \{u \in O_j\}
     \end{equation}
    
    For such l, using the definition of LECs,

    \begin{equation}
        \mathbb{E}[l(U) \mid X=x,Z=z] = \sum_{j=1}^{\mid O \mid} \alpha_j  P(U \in O_j \mid X=x,Z=z) =\sum_{j=1}^{\mid O \mid}  \alpha_j \pi_j(x,z)
    \end{equation}
    
   If completeness holds, then injectivity holds over all of $L^2(U)$. Since $R \subset L^2(U)$, completeness also implies injectivity on $R$. Therefore, completeness distinguishes all LEC-constant functions, but it also requires much  more: it must distinguish functions that are within each LEC.\\
    Next, we show that the rank condition also gives injectivity on $R$, without requiring full completeness. 
    Given distinguishing environment set $Z^*=\{z_1,\dots,z_e\}$, if 
    \begin{equation}
         \mathbb{E}[l(U) \mid X=x,Z=z_i] = 0
    \end{equation}
    for every $z_i \in Z^*$, then
    \begin{equation}
        \Pi(x) \alpha=0,
    \end{equation}
        where $\alpha=(\alpha_1, \dots, \alpha_{\mid O\mid})^T$. By Definition 3, $\Pi(x)$ has full rank $\mid O\mid$ for almost every $x$. Therefore, $\alpha=0$, and hence $l(U)=0$ almost surely for all $l \in R$.  Thus, the rank condition ensures injectivity on the LEC-constant subspace $R$. However, the cross-domain rank condition does not imply completeness.

    Assume proxies are imperfect, so there are some LEC $O_k$ that contains at least two latent values $\{u_1,u_2\}$. Suppose:
    
    \begin{equation}
        P(U=u_1 \mid X=x,Z=z)=  P(U=u_2 \mid X=x,Z=z) \quad \text{for all} \;(x,z)
    \end{equation}

Define a function $l$ supported on $O_k$ by $l(u_1)=1$, $l(u_2)=-1$ and $l(u)=0$ for all $u \notin \{u_1,u_2\}$.
    Then $\mathbb{E}[l(U) \mid X=x,Z=z] =1. P(u_1 \mid x,z)-1. P(u_2 \mid x,z)=0$
    for all $(x,z)$. Thus, completeness fails.
    However, the aggregated LEC mixture weight is:

    \begin{equation}
        \pi_k(x,z)=P(U \in O_k \mid X=x,Z=z) =P(u_1 \mid x,z) + P(u_2 \mid x,z)
    \end{equation}
    
    which can be linearly independent from the mixture weights of other LECs. Thus, the LEC rank condition holds while completeness fails.
  Therefore, the  cross-domain rank condition is weaker from a LEC-level perspective: it only requires separation across LECs, while completeness requires separation over the full latent space, including within-LEC variation. 
    
        \end{proof}

\subsection{Proof of Theorem \ref{th:exbrg}} \label{sc:exbrg}

\begin{proof}
    \textbf{Existence:}
   Fix a distinguishing environment set $Z^{\star}$ (Definition \ref{df:DE}) and let $r(x,z)= \mathbb{E}[Y \mid X=x,Z=z]$. By Lemma \ref{lm:CMEDEC}, for each $z \in Z^{\star}$ and almost surely $x$,

   \begin{equation}
       p(w \mid x,z) =\sum_{j=1}^{K} \pi_j(x,z) p_j(w), \quad K=\mid O \mid.
   \end{equation}

For any measurable function $h(x,w)$, the conditional expectation operator satisfies:

\begin{equation}
    (\mathcal{T}_zh)(x)=\int_{\mathcal{W}} h(x,w)p(w\mid x,z) dw =\sum_{j=1}^{K} \pi_j(x,z)\int_{\mathcal{W}} h(x,w)p_j(w)
    \label{eq:conz}
\end{equation}

Define LEC moment functions:

\begin{equation}
    g_j(x):= \int_{\mathcal{W}} h(x,w) p_j(w)dw, \quad j=1, \dots, K, 
\end{equation}

and write $(g_1(x), \dots, g_k(x))^{\top}$. 

For fixed $x$, stacking Equation \ref{eq:conz} over $z_l \in Z^{\star}$ to obtain the linear system:

\begin{equation}
    r_x=\Pi(x)g(x),
\label{eq:rpi}
\end{equation}

where $r_x := (r(x,z_l))_{l=1}^{e} \in \mathbb{R}^e$ and  $\Pi(x) \in \mathbb{R}^{e \times K} $ has entries $\Pi_{lj}(x)=\pi_j(x,z_l)$.

By the compatibility condition in Assumption \ref{as:LECMom}, for almost surely $x$ there exists a measurable $g(x) \in \mathbb{R}^{K}$ satisfying Equation \ref{eq:rpi}. Moreover, by Assumption \ref{as:LECMom}.3 $g_j \in L_2(\mathcal{X})$ for all $j$.
Next, by Assumption \ref{as:LECMom}.2, $p_j \in L_2(\mathcal{W})$and gram matrix

\begin{equation}
    G \in \mathbb{R}^{K \times K}, \quad G_{jk}:= \int_{\mathcal{W}}p_j(w)p_k(w)dw,
\end{equation}

is invertible. Define  the coefficient vector $\alpha(x):=G^{-1}g(x)$ and construct

\begin{equation}
    h_0(x,w):=\sum_{j=1}^{K} \alpha_j(x)p_j(w),
\end{equation}

we first verify that $h_0 \in L_2(\mathcal{X} \times \mathcal{W})$. Since $g \in L_2(\mathcal{X})^K$ and $a(x)=G^{-1} g(x)$, we have $a_j \in L_2(\mathcal{X})$ for all $j$. In addition, each $p_j \in L_2(\mathcal{W})$  and $K$ is finite.  To show $h_0 \in L_2(\mathcal{X} \times \mathcal{W})$:

\begin{equation}
\begin{aligned}
  &  \mid \mid h_0\mid \mid^2_{L_2(\mathcal{X} \times \mathcal{W})} = \int_{\mathcal{X}} \int_{\mathcal{W}} \mid \sum_{j=1}^{K} a_j(x)p_j(w)|^2 dwdx \leq \\&
K \sum_{j=1}^{K}  \int_{\mathcal{X}} \int_{\mathcal{W}} |a_j(x)|^2|p_j(w)|^2 dwdx=  K \sum_{j=1}^{K} ||a_j(x)||_{L_2(\mathcal{X})}^2||p_j(w)||_{L_2(\mathcal{W})}^2 <\infty. 
\end{aligned}
\end{equation}

where the first inequality is derived from  Cauchy-Schwarz theorem \citep{young1988introduction} and the second equality is application of Fubini's theorem \citep{durrett2019probability}.

Now we compute the moments of $h_0$. For each $k \in \{1, \dots , K\}$,

\begin{equation}
    \int_{\mathcal{W}}h_0(x,w)p_k(w)dw= \sum_{j=1}^{K} a_j(x) \int_{\mathcal{W}}p_j(w)p_k(w)dw =(G a(x))_k =g_k(x)
    \label{eq:brgg}
\end{equation}

By replacing $h$ in Equation \ref{eq:conz} with $h_0$ in Equation \ref{eq:brgg}, for each $z_l \in Z^{\star}$ we have:

\begin{equation}
    (\mathcal{T}_{z_l}h_0)(x)= \sum_{j=1}^{K} \pi_j(x,z_l)g_j(x)=(\Pi(x)g(x))_l=r(x,z_l),
\end{equation}

where the last equality follows from Equation \ref{eq:rpi}. Hence, $\mathcal{T}_zh_0=r(.,z)$ for all $z \in Z^{\star}$. Therefore,

\begin{equation}
    \mathcal{T}_{joint}h=r
\end{equation}

As a result, at least one bridge function exist.

\textbf{Set identification:}
Let $h_1$ and $h_2$ be two solutions of $\mathcal{T}_{joint}h=r$. Then

\begin{equation}
   \mathcal{T}_{joint}h_1=r=\mathcal{T}_{joint}h_2 \Rightarrow \mathcal{T}_{joint}(h_2-h_1)=0
\end{equation}

So  $h_2-h_1 \in ker(\mathcal{T}_{joint})$. Conversely, for any let $d \in ker(\mathcal{T}_{joint})$, we have

\begin{equation}
   \mathcal{T}_{joint} (h_0+d)=r
\end{equation}

Thus the  solution set is:

\begin{equation}
    \mathcal{H}(r)=\{h:\mathcal{T}_zh=r(.,z) (\forall z \in Z^{\star})\} =h_0+ker(\mathcal{T}_{joint})=h_0+\bigcap_{ z \in Z^{\star}} ker(\mathcal{T}_z)
\end{equation}

\end{proof}

\subsection{Proof of Theorem \ref{th:Point}}
\label{sc:Point}

 \begin{proof}

     By Theorem \ref{th:exbrg}, $\mathcal{T}_{joint}h=r$ has at least one solution in  $
    \mathcal{H}(r)=\{h:\mathcal{T}_{joint}h=r\}$. Based on   $\mathcal{T}_{joint}h =(\mathcal{T}_z h)_{z \in Z^{\star}}$, $\mathcal{T}_{joint}h=r$  is equivalent to
    
    \begin{equation}
       \mathcal{T}_zh(.)=r(.,z) \qquad z \in Z^{\star}
    \end{equation}
    
    To show point identification of predictor, given any $h_1,h_2 \in \mathcal{H}_r$, define $\Delta h=h_2 -h_1$. Then
    
   \begin{equation}
      \mathcal{T}_{joint}\Delta h=\mathcal{T}_{joint}h_1 -\mathcal{T}_{joint}h_2=r-r=0
            \label{eq:Ez}
    \end{equation}

Therefore, $\Delta h \in ker(\mathcal{T}_{joint})$. By definition of $\mathcal{T}_{joint}$, this is equivalent to:

    \begin{equation}
        \mathcal{T}_z \Delta h =0 \qquad \forall z \in Z^{\star}
        \label{eq:EDelta}
    \end{equation}

By Lemma \ref{lm:CMEDEC}, the conditional distribution proxy $W$ has a mixture representation. 

\begin{equation}
    P(W|X=x,Z=z) = \sum_{j=1}^{|O|} \pi_j(x,z) P(W|U \in O_j)
\end{equation}

Recall the definition of the conditional expectation operator $\mathcal{T}_z$ acting on any function h:

\begin{equation}
    (\mathcal{T}_zh)(x)=\int_w h(x,w) P(w|x,z) dw
\end{equation}

Substituting the mixture representation from Lemma \ref{lm:CMEDEC} into this integral:

\begin{equation}
    (\mathcal{T}_z h) (x) = \int_{\mathcal{W}} h(x,w) \Big( \sum_{j=1}^{|O|} \pi_j(x,z) P(w|U \in O_j) \Big) dw
\end{equation}

By the linearity of the integral, we can exchange the sum and integral:

\begin{equation}
    (\mathcal{T}_zh)(x) = \sum_{j=1}^{|O|}  \pi_j(x,z) \underbrace{\int_\mathcal{W} h(x,w)  P(w|U \in O_j) dw}_{\mathcal{T}_{O_j}h}
\end{equation}

Thus, the operator itself decomposes as:

\begin{equation}
   \mathcal{T}_z =   \sum_{j=1}^{|O|} \pi_j(x,z)\mathcal{T}_{O_j}
\end{equation}

where $\mathcal{T}_{O_j}$ is the operator defined by the conditional distribution of the proxy given the $j$-th equivalent classes. Similarly based on invariance of causal structure (Assumption \ref{ass: shift}) we have:

\begin{equation}
   \mathcal{T}_{S+1} =   \sum_{j=1}^{|O|} \pi_j(x,S+1) \mathcal{T}_{O_j}.
\end{equation}

By Definition \ref{df:DE} for almost every $x$, the matrix of distinguishing environment set $\Pi (x)=[\pi_j(x,z_l)]_{l,j}, l=1,\dots,e, j=1,\dots,|O|$ has  full rank $K=|O|$. Since dimension of mixture space is $K$, the vectors $\{\pi(x,z):z \in Z^{\star}\}$
 span entire space $\mathbb{R}^K$. As a result, the mixture weight vector $\pi(x,S+1) \in \mathbb{R}^K$ lies in the linear span of source mixture weights. Therefore, there is coefficient $\{\gamma_z(x)\}_{z \in Z^{\star}}$ such that:

 \begin{equation}
    \mathcal{T}_{S+1}=\sum_{z \in Z^{\star}} \gamma_z(x) \mathcal{T}_z
 \end{equation}

Following the same steps leading to Equations \ref{eq:Ez}
and \ref{eq:EDelta}  for target domain we have:

\begin{equation}
      \mathcal{T}_{S+1} \Delta h =   \sum_{z \in Z^{\star}} \Big(\gamma_z(x) \mathcal{T}_z\Big) \Delta h =\sum_{z \in Z^{\star}} \gamma_z(x) \underbrace{\Big(\mathcal{T}_z \Delta h \Big)}_0 =0\
\end{equation}

This implies $\mathcal{T}_{S+1}h_1=\mathcal{T}_{S+1}h_2$. Thus, the predictor $\mathbb{E}[Y|X,Z=S+1]$ is uniquely identified.
    \end{proof}

\section{Derivation of acquisition function} \label{sc:DerAct}
Both CME and bridge training are performed via ridge regression, which corresponds to a Gaussian  process \citep{kanagawa2018gaussian}. As a result, the output  also follows a joint Gaussian distribution. Therefore, we can replace the information gain $[I(y;y_*|d,d_*)]$ with its closed form $[I(y;y_*|d,d_*)]=\mathbb{H}(y|d)+\mathbb{H}(y_*|d_*)-\mathbb{H}(y,y_*|d,d^*)$ where $\mathbb{H}$ is entropy. Each entropy has a closed form as follows:

\begin{equation}
    \begin{aligned}
        \mathbb{H}(y|d) =\frac{1}{2}\log(2 \pi e \mathbb{V}(y|d)) \\
        \mathbb{H}(y_*|d_*)= \frac{1}{2}\log(2 \pi e \mathbb{V}(y_*|d_*)) \\
         \mathbb{H}(y;y_*|d,d_*) = \frac{1}{2} \log ((2 \pi e)^2 det(\Sigma))
    \end{aligned}
\end{equation}

Where $\Sigma=\begin{pmatrix}
\mathbb{V}(y|d) & Cov(y,y_*|d,d_*)\\
Cov(y,y_*|d,d_*) & \mathbb{V}(y_*|d_*)
\end{pmatrix}$ and $\mathbb{V}$ is variance and $Cov$ is covariance matrix.
Then, mutual information can be obtained by:

\begin{equation}
\begin{aligned}
     I(y;y_*|d,d_*) = [\frac{1}{2}\log(2 \pi e \mathbb{V}(y|d))] + [\frac{1}{2}\log(2 \pi e \mathbb{V}(y_*|d_*))] -[\frac{1}{2}\log((2 \pi e)^2 det(\Sigma))] =\\
     \frac{1}{2} \log [\frac{\mathbb{V}(y|d) \mathbb{V}(y_*|d_*)}{det(\Sigma)}]= \frac{1}{2} \log [\frac{\mathbb{V}(y|d) \mathbb{V}(y_*|d_*)}{\mathbb{V}(y|d) \mathbb{V}(y_*|d_*)-Cov(y,y_*|d,d_*)}]
\end{aligned}
\label{eq:infgan}
\end{equation}

Define $\rho(d,d_*)=\frac{Cov(y,y_*|d,d_*)}{\sqrt{\mathbb{V}(y|d)\mathbb{V}(y_*|d_*)}}$ and therefore $Cov(y,y_*|d,d_*)^2=\rho(d,d_*)^2 \mathbb{V}(y|d)\mathbb{V}(y_*|d_*)$. By substituting  $Cov(y,y_*|d,d_*)^2$ in Equation \ref{eq:infgan} we get:

\begin{equation}
     I(y;y_*|d,d_*) = \frac{1}{2} \log [\frac{\mathbb{V}(y|d) \mathbb{V}(y_*|d_*)}{\mathbb{V}(y|d) \mathbb{V}(y_*|d_*)-\rho^2(d,d_*) \mathbb{V}(y|d)\mathbb{V}(y_*|d_*)}]
\end{equation}

Then, by canceling out similar terms, we have:

\begin{equation}
    \begin{aligned}
         I(y;y_*|d,d_*) = -\frac{1}{2} \log (1-\rho^2(d,d_*)) 
    \end{aligned}
\end{equation}

\section{Guarantees for CME uncertainty acquisition}\label{sc:GarAcq}
 In this section, we show that the acquisition function based on CME uncertainty recovers resolvable rank.

Based on \citep{kanagawa2018gaussian}, in the case when observations are noise-free  and the Gram matrix 
 $\mathcal{K}_{XZ}$ is invertible ,  the standard deviation of posterior  ($\sigma_{\mu}(x,z)= \sqrt{k(x,x)  k(z,z)-(\Phi_X(x) \odot \Phi_Z(z))^{\top} (\mathcal{K}_{XZ})^{-1} (\Phi_X(x) \odot \Phi_Z(z))}$),  will  be equal to worst case error  $sup_{f \in \mathcal{H}_k, \mid\mid f\mid \mid_{\mathcal{H}_k} \leq 1}\big(\sum_{i=1}^n \omega_i(x,z)f(x_i,z_i)-f(x,z)\big )$ where $(\omega_1 (x,z), \dots, \omega_n(x,z) ) = \mathcal{K}^{-1}_{XZ} (\Phi_X(x) \odot \Phi_Z(z))$. 
In CME setting, the goal of ridge regression is learning $\mathbb{E}[g(W) \mid X=x,Z=z] = <g, \mu_{x,z}>$ so the worst case error becomes $sup_{ \mid\mid g\mid \mid_{\mathcal{H}_{\mathcal{W}}} \leq 1}\mid <g, \mu_{x,z}> - \sum_{i=1}^n \omega_i(x,z)<g, \mu_{x_i,z_i}>\mid = sup_{ \mid\mid g\mid \mid_{\mathcal{H}_{\mathcal{W}}} \leq 1}\mid  \mu_{x,z} - \sum_{i=1}^n \omega_i(x,z) \mu_{x_i,z_i}\mid $. Intuitively, this means whenever the posterior variance is zero, the conditional mean embeddings $\mu_{x,z}$ lie in the linear span of previously observed embeddings and introduce no new independent direction. Whereas a positive posterior variance certifies  the presence of a new linearly independent component.
In practice, observed data contains noise, and a regularization term $\lambda_{CME}$ is used  when $\mathcal{K}_{XZ}$ is not invertible. However, in such conditions, we can assume a relaxed version of this condition holds by considering an error bound.

\begin{assumption} [Approximate variance span correspondence]
    For a fixed $x \in \mathcal{X}$, let $\mathcal{M}_{r-1}:=span\{\mu_{x,z_1}, \dots, \mu_{x,z_{r-1}}\} \subset \mathcal{H}_{\mathcal{W}}$ be the span of CMEs selected up to round $r-1$. Assume there exist constants $\epsilon \geq 0$  and $\eta >0$ such that for all rounds r and and all candidates $z \in \mathcal{Z}_{\mathcal{PL}}$:
    \begin{equation}
    \begin{aligned}
        \mu_{x,z} \in \mathcal{M}_{r-1} \Rightarrow \sigma^2_{\mu}(x,z) \leq \varepsilon  \\
          dist(  \mu_{x,z}, \mathcal{M}_{r-1}) \geq \eta  \Rightarrow \sigma^2_{\mu}(x,z) \geq \varepsilon +\eta 
    \end{aligned}
    \end{equation}

Here $dist(\mu, \mathcal{M}):= inf_{m \in \mathcal{M}} \mid\mid \mu-m\mid\mid_{\mathcal{H}_{\mathcal{W}}}$ denotes the RKHS distance to subspace $\mathcal{M}$.  
\label{ass:VSC}
\end{assumption}

\begin{remark}[Regularization constraint]
We note that satisfying Assumption \ref{ass:VSC} implicitly constrains the regularization parameter $\lambda_{CME}$ in CME posterior uncertainty in  Equation \ref{eq:CMEUnc}. Specifically, $\lambda_{CME}$ must be sufficiently small such that it does not dampen the posterior variance of orthogonal components below the separation margin $\varepsilon+\eta$. In our experiments, we set $\lambda_{CME}$ as a hyperparameter and tune it to maintain this sensitivity.  
\end{remark}
While the identification results allow an arbitrary set of environments, in practice, the learner can only select from a finite set of available environments. We therefore analyze a pool-based policy that selects z from the candidate pool $\mathcal{Z}_{\mathcal{PL}}$. The following theorem shows that  by using  CME uncertainty,  the resolvable rank within this pool is recoverable, which is the finite-sample analogue of the distinguishing set rank used for identification.

\begin{definition} [Resolvable pool rank]
Let 
    $\zeta_{\eta}(x)$ be largest integer $\zeta$  for which there exist environments $z_1, \dots , z_{\zeta} \in \mathcal{Z}_{\mathcal{PL}}$ such that for each $j=1,\dots, \zeta$,
    \begin{equation}
        dist(\mu_{x,z_j}, span\{\mu_{x,z_1},\dots, \mu_{x,z_{j-1}}\}) \geq \eta
    \end{equation}

Equivalently $\zeta_{\eta}(x)$ is the maximum number of linearly independent directions in $\mathcal{H}_{\mathcal{W}}$ that are separable from 
    each other by a margin $\eta$.
    
\end{definition}

\begin{theorem} [PQAL acquisition and resovable rank]
For fixed $x \in \mathcal{X}$, define data pool information rank

\begin{equation}
    rank_{\mathcal{PL}}(x) := dim(span\{\mu_{x,z}: z\in \mathcal{Z}_{\mathcal{PL}}\}) < \infty 
\end{equation}

  Consider Assumption \ref{ass:VSC}  holds. Consider the acquisition policy:

\begin{equation}
    z_r \in arg \; max_{z \in \mathcal{Z}_{\mathcal{PL}}} \sigma^2_{\mu}(x,z).
\end{equation}

  Then the selected CMEs satisfy

\begin{equation}
    dim(span\{\mu_{x,z_1}, \dots, \mu_{x,z_r} \}) \geq \min \{r, \zeta_{\eta}(x)\}.
\end{equation}

  In particular, after $\zeta_{\eta}(x)$ rounds, the selected set covers all CME directions in the pool that are resolvable at  scale $\eta$.

  \begin{proof}
      Let $\mathcal{M}_{r-1}:= span\{\mu_{x,z_1}, \dots, \mu_{x,z_{r-1}}\}$. If $dim (\mathcal{M}_{r-1}) < \zeta_{\eta}(x)$, then by definition of $\zeta_{\eta}(x)$ there exist some $\bar{z} \in \mathcal{Z}_{\mathcal{PL}}$ such that:
      
      \begin{equation}
          dist(\mu_{x,\bar{z}}, \mathcal{M}_{r-1}) \geq \eta
      \end{equation}

By Assumption \ref{ass:VSC}, this implies $\sigma^2_{\mu}(x,\bar{z}) \geq \varepsilon + \eta $. Since $z_r$ maximizes $\sigma_{\mu}^2(x,.)$, we have $\sigma^2_{\mu}(x,z_r) \geq \varepsilon + \eta$. On the other hand, if $\mu_{x,z_r} \in \mathcal{M}_{r-1}$, Assumption \ref{ass:VSC} would imply $\sigma^2_{\mu}(x,z_r) \leq \varepsilon$, contradicting $\sigma_{\mu}^2(x,z_r) \geq \varepsilon + \eta$. Therefore $\mu_{x,z_r} \notin \mathcal{M}_{r-1}$, and hence $dim(\mathcal{M}_r)=dim(\mathcal{M}_{r-1})+1$. If instead $dim(\mathcal{M}_{r-1}) \geq \zeta _{\eta}(x)$, the lower bound $dim(\mathcal{M}_r) \geq \zeta_{\eta}(x)$ is immediate. The claim follows by induction.
  
  \end{proof}
\end{theorem}

\section{Algorithm}
\FloatBarrier
\begin{algorithm}[H]
\caption{Proximal Quasi Bayesian Active learning (PQAL)} 
\begin{algorithmic}[1]
   \Require $\mathcal{D}_{\mathcal{LB}}=\{(x_i,w_i,z_i,y_i)\}^{m}_{i=1}$: Data set with labels from source domains,
 $\mathcal{D}_{\mathcal{PL}}=\{(x_i,z_i,)\}^{n}_{i=1}$: Pool of data samples without label and proxy
 Budget $R$, Empty set of samples with proxy $\mathcal{D}_{\mathcal{PR}}$,  Empty set of samples with proxy and label $\mathcal{D}_{\mathcal{TLB}}$, number of proxy queries per round $v_{p}$, number of labeled  queries per round $v_{lb}$
 \Ensure 
 Bridge function parameter $\alpha$, CMEs $\{\mu_{W|x,z}\}_{z \in \mathcal{Z}}$, And optimal predictor $\mathbb{E}[Y|X=x,Z=z]$

 \State \textbf{Initial training:}
 \State Train CMEs for each environment by using closed form of Equation \ref{eq:CME}
 \State  Train bridge by using its closed form \ref{eq:brgcl} and obtain bridge parameter $\alpha_0$
\While{Budget $r\leq R$}  
\If {r=1}
 \State Select $v_p$ and $v_{lb}$ samples randomly
\Else
 \State Select $v_p$ and $v_{lb}$ samples samples based on acquisition function in Equation \ref{eq:CMEUnc}
 \State ($(x,z)\in \argmax_{(x,z)}=\sigma^2_{\mu,r}(x,z)$)
 \EndIf
 \State Query $w$ from external source of information for $v_p$ selected samples $D^r_{\mathcal{PR}}=D^r_{\mathcal{PR}}\cup \{(x_i,z_i,w_i)\}_{i=1}^{v_p}$
  \State Query $w$ , $y$ from external source of information for $v_{lb}$ selected samples \State $D^r_{\mathcal{TLB}}=D^r_{\mathcal{TLB}}\cup \{(x_i,z_i,w_i,y_i)\}_{i=1}^{v_{lb}}$
 \State Remove $v_p$ and $v_{lb}$ samples from the pool of candidate samples \State$D_{\mathcal{PL}}= D_{\mathcal{PL}} \setminus	\{(x,z): (x,z) \in( D^r_{\mathcal{PR}} \cup  D^r_{\mathcal{TLB}})\} $
\State Train CME $\mu^{(r)}_{W|X=x,Z=z}$ in target domain by using equation \ref{eq:CMEAdp}
\State  Adapt bridge function parameter $\alpha^{(r)}$  by  Equation \ref{eq:losses}
\EndWhile
\State \Return
$\alpha^{(r)}$,  $\{\mu^{(r)}_{W|x,z}\}_{z \in \mathcal{Z}}$
\end{algorithmic}
\label{Al:algorithm}
\end{algorithm}

\FloatBarrier

\section{Experiment details}

In this section we introduce each baseline with details.

\subsection{Baselines}\label{sc:Base}
\begin{itemize}
    \item \textbf{Proxy-DA}:
\citet{tsai2024proxy} assume completeness and prove the trained bridge in source domain can be used on target domains ($h_s=h_t$). As a result, only updating CMEs  in target domain is enough to adapt the predictor and get robust predictor.
  \item \textbf{SC}:
\citet{prashant2025scalable} uses discrete proxies $W$ to capture the structure of unobserved variables. This approach first trains an encoder to learn the latent distribution $P(U|X)$ by using $(X,W)$. A mixture of experts is then trained on source data, where predictions are combined by using inferred latent at inference time, the model estimates the target latent marginal by comparing the distribution of latent components for the queried target covariates against the source. These estimates are used to recalibrate the mixture weights, allowing the model to adapt its predictions to target domain shift.

\item \textbf{Few-shot-ERM:}
In few-shot Empirical risk minimization (ERM), we concatenate data from source and target domains and train a MLP over concatenation of current source samples and few queried target samples $\{(x,y)\}_{(x,y) \in \mathcal{D}_{\mathcal{LB}} \cup \mathcal{D}_{\mathcal{TLB}}}$  and predictor is obtained by minimizing loss over these samples. 

\item\textbf{Oracle:}
Oracle baseline has access to all target samples from candidate pool.  It is once trained on source labeled samples plus all labeled samples from candidate pools.
We then fine tune shred source bridge based on loss source and target losses in Equation \ref{eq:losses} $
    (\hat{\alpha}^{(r)}= \argmin_{\alpha} (\mathcal{L}_{source}(\alpha)+ \lambda_{tgt} \mathcal{L}^{(r)}_{target} (\alpha) + \lambda_{reg}\mid\mid \alpha -\alpha^{(0)}\mid \mid_2^2)$.

\end{itemize}

\subsection{Compared Acquisition strategies}\label{sc:AcStr}

\begin{itemize}
    \item \textbf{Representative $x$ (RPX):}
    We use an expected error reduction (EER) inspired  surrogate that favors representative samples in covariate space. Concretely, given the current unlabeled target pool $\{x_i\}_{i=1}^{n}$, we form the kernel gram matrix $\mathcal{K}_{X} \in \mathbb{R}^{n \times n}$. Each candidate $x_i$ is scored by its (regularized) global influence,
    \begin{equation}
        s_i=\frac{\sum_{j=1}^n k^2(x_i,x_j)}{k(x_i,x_i)+\lambda}
    \end{equation}

    where $\lambda>0$ is a ridge parameter. Intuitively nominator $\sum_{j=1}^n k^2(x_i,x_j)$ measures how strongly $x_i$ is correlated with the remainder of the pool in kernel space. Thus, points with high scores are central and capture the structure that is shared by many samples in the candidate pool. Selecting these points tends to a reduction in the aggregated posterior variance and  a  tractable approximation for EER.
        \item \textbf{Z-uncertainty (environment coverage acquisition):}
        This strategy prioritizes coverage of underrepresented environments. At each active learning round, we compute the empirical frequency of each environment value $z$ among current labeled samples. Then we   assign each candidate $(x_i,z_i)$ a score that is inversely proportional to the number of labeled points that are already collected from the same environment as candidate $s_i=\frac{1}{n_{z_i}+1}$, where $n_{z_i}$ denotes the current labeled count for environment $z_i$. This acquisition favors environments that have been observed less.  
\item \textbf{Random:} This strategy randomly selects a point from the candidate pool.
\end{itemize}

\subsection{Experiment setup} \label{sc:setup}
For all experiments, we use $35$ samples from each source environment at the beginning of the active learning loop and  query $5$ samples at each iteration. For PQAL and Few-shot-ERM, $2$  of the $5$ samples are equipped with label $y$ and proxy $w$, and the remaining $3$  receive only  proxy $w$.   
In the following, we  mention the specific settings for each experiment. For the robustness and acquisition experiments, we set $B=4$.

\begin{enumerate}
    \item \textbf{Effect of imperfect proxies on  LECs:}
We ran experiments with varying numbers of source environments ($2,3,4$) and a total budget of $45$ queries across $6$ different seeds. We use following beta distributions for source distributions $\{(2,10),(4,8),(8,4),(10,2)\}$ we fix target distributions to $(6,6)$.

 \item \textbf{Robustness under varying degrees of distribution shift on synthetic datasets:}
We ran  experiments with $2$ source environments and a total budget of $60$ queries across $4$ different seeds. We use following beta distributions for source distributions $\{(2, 4), (2.1, 3.9)
\}$,  we vary degrees of shift in target distributions by using $\{ (8.0, 12.0),   (6.0, 6.0),   (5.0, 3.333),  (3.0, 1.286),   (2.0, 0.5)     \}$ beta distributions.

 \item \textbf{Robustness under varying degrees of distribution shift on dSprites dataset:}
We ran  experiments with $2$ source environments and a total budget of $60$ queries across $6$ different seeds. We use following beta distributions for source distributions $\{(2, 4), (2.1, 3.9)
\}$,  we vary degrees of shift in target distributions by using $\{(10,10),(15,5),(25,3),(40,1), (60,0.5)\}$ beta distributions.
Figure \ref{fig:dspvis} shows the shift in image orientation  from source to target environment. 
\item \textbf{Robustness under varying degrees of distribution shift on IHDP dataset}:
We ran experiments with 2 source environments and a total budget of $60$ queries across $6$ different seeds. We partition the data into the domain $Z$ based on  birth weight to simulate distinct clinical environments. source $1$ consists of infants with normal/high birth weight, source 2 includes Medium/low birth weight infants, and the target domain is comprised of extremely low birth weight infants. The latent confounder $U$ is conditioned on these weight-based partitions to model distinct underlying risk profiles. We use the beta distributions $\{(2,5),(5,2)\}$ for the source environments. We quantify the degree of shift by varying the directional mean of target distributions using parameters $\{(3,4),(1,5,5),(1.33,5.17),(1.17,5.33), (1,5.5)\}$  which progressively shifted the expected value of $U$ toward the extreme tail to model unobserved clinical risks.
 \item \textbf{Robustness under varying degrees of distribution shift on Folktables dataset}:
 We conducted  experiments using 2 source environments  (California and Texas) and a total budget of $60$ queries across $6$ different seeds. We partition the data into domain $Z$ based on US states to simulate distinct regional socioeconomic environments, conditioning the latent confounder $U$ on these state-level assignments.  We use the beta distributions $\{(2,5),(5,2)\}$ for the source environments. We quantify the degree of shift by varying the directional mean of target distributions using parameters $\{(3,4),(1,5,5),(1,5.5),(0.5,6), (0.2,6.5)\}$ for states Nevada, New York, Pennsylvania, Ohio, and Alabama, respectively. As a result,  the expected value of $U$ is shifted toward  the extreme tail to model unobserved regional socioeconomic pressures.

\begin{figure}
    \centering
    \includegraphics[width=\linewidth]{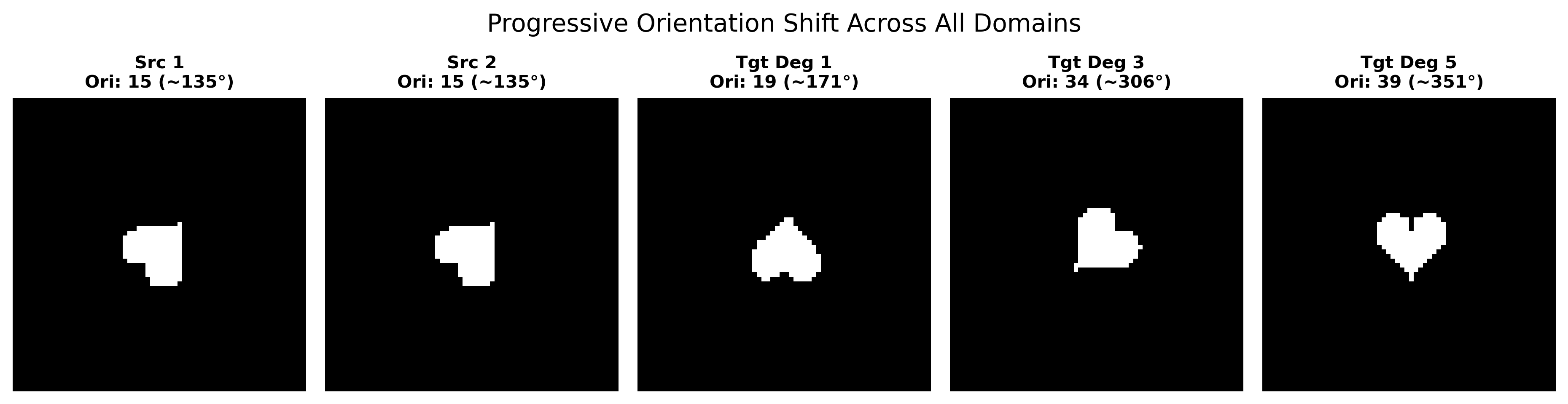}
    \caption{Change in latent variable orientation from source to target domains}
    \label{fig:dspvis}
\end{figure}
 \item \textbf{Effect of acquisition function:}
We ran experiments with $2$ source environments and a total budget of $45$ queries across $4$ different seeds. We use following beta distributions for source distributions $\{(2, 4), (2.1, 3.9)
\}$, and a target distribution $\{ (2.0, 0.5)     \}$.

\end{enumerate}

Note that the source and target distributions in the first experiment is differ from those in the the second and third experiments, as we wanted to maximize latent heterogeneity to ensure any reduction in rank is due to proxy imperfection rather than environment similarity. Conversely, for the robustness and acquisition experiments, we selected  closer source environments, thereby isolating the challenge of adapting to significant distribution shifts in the target.

\end{document}